\documentclass{article}
\usepackage[T1]{fontenc}

\usepackage{amsmath,mathtools,bm,amssymb,amsfonts,dsfont}
\usepackage{amsthm}

\usepackage[ruled,vlined,linesnumbered]{algorithm2e}

\SetKwInOut{Input}{Initialization}
\SetKwInOut{Output}{Output}
\SetArgSty{upshape}

\usepackage[dvipsnames]{xcolor} 

\definecolor{linkBlue}{HTML}{2574A9} 

\usepackage{xspace}
\usepackage[
  backend=biber,
  style=alphabetic,
  natbib=true
]{biblatex}

\usepackage{thm-restate}
\usepackage{upgreek}
\usepackage
[
    bookmarks = true,                 
    bookmarksopen = false,            
    bookmarksnumbered = true,         
    pdfstartpage = 1,                 
    colorlinks = true,                
    allcolors = linkBlue,              
]
{hyperref}
\usepackage[nameinlink]{cleveref}

\newtheorem{theorem}{Theorem}

\newtheorem{definition}{Definition}

\crefalias{lem}{Lemma}

\crefname{case}{case}{cases}
\crefname{algocf}{algorithm}{algorithms}
\crefname{assump}{assumption}{assumptions}
\crefname{condition}{condition}{conditions}

\let\originalleft\left
\let\originalright\right
\renewcommand{\left}{\mathopen{}\mathclose\bgroup\originalleft}
\renewcommand{\right}{\aftergroup\egroup\originalright}

\makeatletter
\def\NAT@spacechar{~}
\makeatother

\makeatletter
    \def\IfEmptyTF#1%
    {%
        \if\relax\detokenize{#1}\relax%
            \expandafter\@firstoftwo%
        \else%
            \expandafter\@secondoftwo%
        \fi%
    }
\makeatother

\newcommand{\ie}{i.\,e.\xspace}
\newcommand{\eg}{e.\,g.\xspace}
\newcommand{\ONEMAX}{\textsc{OneMax}\xspace}
\newcommand{\OneMax}{\ONEMAX}
\newcommand{\BV}{\textsc{BinVal}\xspace}
\newcommand{\LO}{\textsc{LeadingOnes}\xspace}
\newcommand{\ronemax}{\ensuremath{r}\text{-}\textsc{OneMax}\xspace}
\newcommand{\rLO}{\ensuremath{r}\text{-}\textsc{LeadingOnes}\xspace}
\newcommand{\gonemax}{\text{G\nobreakdash-}\textsc{OneMax}\xspace}
\newcommand{\cga}{cGA\xspace}
\newcommand{\rcga}{$r$-cGA\xspace} 
\newcommand{\rumda}{$r$\nobreakdash-UMDA\xspace}

\newcommand{\globalOneMax}{\textsc{GOM}\xspace}

\newcommand{\cdrift}{c_{\textup{drift}}\xspace}
\newcommand{\cstar}{c^*\xspace}

\newcommand{\cphase}{c_{\textup{phase}}\xspace}
\newcommand{\cstop}{c_{\textup{stop}}\xspace}
\newcommand{\tstop}{t_{\textup{stop}}\xspace}

\newcommand{\xmin}{x_{\min}}

\DeclareDocumentCommand{\freq}{m m m}{%
    \boldsymbol{p}%
    \IfEmptyTF{#1}{}{^{(#1)}}%
    \IfEmptyTF{#2}{}{_{#2\IfEmptyTF{#3}{}{, #3}}}
}
\DeclareDocumentCommand{\freqTwo}{m m m}{%
    \boldsymbol{q}%
    \IfEmptyTF{#1}{}{^{(#1)}}%
    \IfEmptyTF{#2}{}{_{#2\IfEmptyTF{#3}{}{, #3}}}
}
\DeclareDocumentCommand{\filt}{m}{%
    \mathcal{F}%
    \IfEmptyTF{#1}{}{_{#1}}
}
\DeclareDocumentCommand{\filtTwo}{m}{%
    \mathcal{G}%
    \IfEmptyTF{#1}{}{_{#1}}
}
\DeclareDocumentCommand{\sample}{m m m}{%
    \boldsymbol{x}%
    ^{(\IfEmptyTF{#1}{}{#1, } #2)}%
    \IfEmptyTF{#3}{}{_{#3}}
}
\DeclareDocumentCommand{\selectedSample}{m m m}{%
    \boldsymbol{y}%
    ^{(\IfEmptyTF{#1}{}{#1, } #2)}%
    \IfEmptyTF{#3}{}{_{#3}}
}
\DeclareDocumentCommand{\sampleTilde}{m m m}{%
    \widetilde{\boldsymbol{x}}%
    ^{(\IfEmptyTF{#1}{}{#1, } #2)}%
    \IfEmptyTF{#3}{}{_{#3}}
}
\DeclareDocumentCommand{\sumOfSampleValues}{m m m}{%
    X%
    ^{(\IfEmptyTF{#1}{}{#1, } #2)}%
    \IfEmptyTF{#3}{}{_{\smallsetminus #3}}
}
\DeclareDocumentCommand{\differenceOfSampleValues}{m m}{%
    D%
    \IfEmptyTF{#1}{}{^{(#1)}}%
    \IfEmptyTF{#2}{}{_{#2}}
}
\DeclareDocumentCommand{\eulerE}{o}{%
    \mathrm{e}\IfEmptyTF{#1}{}{^{#1}}
}
\DeclareDocumentCommand{\randomWalkStoppingTime}{m m o}{%
    R^{#1}_{#2\IfEmptyTF{#3}{}{, #3}}
}
\DeclareDocumentCommand{\biasedStoppingTime}{m m o}{%
    \overline{R}^{#1}_{#2\IfEmptyTF{#3}{}{, #3}}
}

\newcommand*{\jStar}{j^\star}
\newcommand{\sDstar}{\mu_i^{(t')}(S_{\kappa+1})}
\newcommand{\sDstarP}[1]{\mu_i^{(#1)}(S_{\kappa+1})}
\newcommand{\sTstar}[1]{\mu_i^{(#1)}(S_{\kappa+2})}

\DeclareMathOperator{\Var}{Var}

\renewcommand{\epsilon}{\varepsilon}
\newcommand{\R}{\mathds{R}}
\newcommand{\N}{\mathds{N}}

\newcommand{\poly}{\mathrm{poly}}
\DeclarePairedDelimiter{\floor}{\lfloor}{\rfloor}
\DeclarePairedDelimiter{\ceil}{\lceil}{\rceil}
\DeclarePairedDelimiter{\abs}{\lvert}{\rvert}
\DeclarePairedDelimiter{\card}{\lvert}{\rvert}
\DeclarePairedDelimiterXPP{\ones}[1]{}{\lvert}{\rvert}{_1}{#1}
\DeclarePairedDelimiterXPP{\zeros}[1]{}{\lvert}{\rvert}{_0}{#1}
\DeclarePairedDelimiterXPP{\zerosBlock}[2]{}{\lvert}{\rvert}{_{0, #2}}{#1}
\DeclarePairedDelimiterXPP{\hammingDist}[2]{d_{\mathrm{H}}}{(}{)}{}{#1, #2}
\providecommand{\given}{} 
\newcommand{\verticalBar}[1][]{%
    \nonscript\;#1\vert
    \allowbreak
    \nonscript\;
    \mathopen{}
}
\NewDocumentCommand{\DefineOperatorWithBar}{mmmm}{
    \DeclarePairedDelimiterXPP{#1}[1]{#2}{#3}{#4}{}{%
        \renewcommand{\given}{\verticalBar[\delimsize]}
        ##1
    }
}
\DefineOperatorWithBar{\set}{}{\{}{\}}
\DefineOperatorWithBar{\prob}{\Pr}{[}{]}
\DefineOperatorWithBar{\expect}{\mathrm{E}}{[}{]}
\DefineOperatorWithBar{\var}{\mathrm{Var}}{[}{]}
\DefineOperatorWithBar{\indic}{\mathds{1}}{\{}{\}}

\DeclarePairedDelimiterXPP{\bigO}[1]{\mathrm{O}}{(}{)}{}{#1}
\DeclarePairedDelimiterXPP{\bigOmega}[1]{\Upomega}{(}{)}{}{#1}
\DeclarePairedDelimiterXPP{\bigTheta}[1]{\Uptheta}{(}{)}{}{#1}
\DeclarePairedDelimiterXPP{\smallO}[1]{\mathrm{o}}{(}{)}{}{#1}
\DeclarePairedDelimiterXPP{\smallOmega}[1]{\upomega}{(}{)}{}{#1}

\newcommand{\ignore}[1]{}

\begin{document}

\title{\textbf{Runtime Analysis of a Compact Genetic Algorithm on a Truly Multi-valued OneMax Function}}

\author{Martin~S. Krejca$^1$ \and Carsten Witt$^2$}
\date{\small
    $1$ Laboratoire d'Informatique (LIX), CNRS, École Polytechnique, IP Paris, France\\\texttt{martin.krejca@polytechnique.edu}\\
    $2$ DTU Compute, Technical University of Denmark, Denmark\\\texttt{cawi@dtu.dtk}
}

\maketitle

\begin{abstract}
    Recently, the runtime analysis of multi-valued estimation-of-distribution algorithms in the framework of Ben Jedidia et al.~(TCS~2024) has made
    significant advancements. However, almost all existing analyses are limited
    to multi-valued objective functions that in each dimension only distinguish between two types,
    also called categories, of values and hence
    can be treated with similar methods as pseudo-Boolean problems. Only recently,
    Adak and Witt~(GECCO~2025) have presented a first runtime analysis of
    a multi-valued compact genetic algorithm (cGA) on the multi-valued OneMax
    function G\nobreakdash--OneMax$\colon \{0,\dots,r-1\}^n \to \N$ defined
    by G\nobreakdash-OneMax$(\bm{x}_1,\dots,\bm{x}_n)=\sum_{i=1}^n \bm{x}_i$ and truly
    depending on all~$r$ categories.
    We improve their runtime result from $\bigO{n r^3 \log^2( n)\log (r)}$ to
    $\bigO{n r \log^3(n)\log^3(r)}$, both for an optimal
    choice of the update strength~$K$. Our result matches, up to polylogarithmic factors, the existing bound for the
    simpler $r$-valued OneMax function depending essentially only on two
    values and analyzed in several previous works. To show the new bound, we
    use improved drift theorems for processes with high self-loop probabilities
    and specifically derived concentration inequalities to analyze how probability mass in the multi-valued cGA moves into successively smaller and smaller intervals of the $r$-valued
    frequency matrix.\\

\end{abstract}

\section{Introduction}
Recently, there has been growing interest in the runtime analysis of estimation-of-distribution algorithms (EDAs) on multi-valued search spaces. Classic EDAs work on the search space $\{0,1\}^n$ of bit
strings of length~$n$, which is a simple and flexible representation.
However, several combinatorial problems, such as routing and permutation
problems, would not naturally be modeled as bit strings. To allow
a more natural representation, recent studies have focused on the
search space $R^n\coloneqq \{0,1,\dots,r-1\}^n$, \ie, where each component of the $n$-dimensional
vector takes one of~$r$ different values. The domain size~$r$ can grow
with the problem dimension~$n$. 

The first runtime analysis of an evolutionary algorithm (EA) for the search space $R^n$ appeared 15 years ago \citep{DoerrEtAlFOGA11} and since then,
several studies of EAs have followed \citep{doerr2012run,doerr2013evolutionary,kotzing2015, lissovoi2014mmas,doerr2018static,Doerr_AAAI_2025}.
However,
only recently have runtime analyzes of EDAs for the multi-valued
search space~$R^n$ appeared. These analyses started with
\citet{BenJedidiaDK24} laying a framework for multi-valued variants
of classic univariate EDAs like the compact genetic algorithm (\cga~\cite{HarikLG99cGA}), univariate marginal distribution algorithm (UMDA~\cite{MuehlenbeinP96UMDA}), and others. Their
work was also the first to conduct a runtime analysis of
a multi-valued EDA; more precisely, they proved that the multi-valued
UMDA called \rumda optimizes an $r$-valued \LO function in
time $\bigO{ n^2 r \log(n)\log^2(r) }$ with high probability for appropriate
choices of the algorithm parameters. Moreover, they gave a lower bound
of similar order. Subsequently, several studies have considered
a multi-valued variant of the \cga; in particular, \citet{HamanoUSMA25},
independently of \citet{BenJedidiaDK24}, proved runtime bounds
for the multi-valued \cga (called \emph{categorical compact GA}
in their work)
on variants of the well-known
\OneMax and \BV problems, including a bound of $\bigO{nr\log^2(nr)}$ on
generalized \OneMax holding
with high probability for an optimized setting of the
algorithm parameter. Independently of \citet{HamanoUSMA25},
\citet{AdakWittPPSN2024} proved a runtime bound for technically the same
multi-valued
\cga, called \rcga as in the framework of \citet{BenJedidiaDK24},
on the same generalized \OneMax function, albeit
resulting in a weaker bound of $\bigO{nr^2 \log^3 (n)\log^2 (r)}$. In a
follow-up work, \citet{AdakW25} improved this bound to $\bigO{nr \log (n)\log(r)}$, matching the bound $\bigO{n\log n}$ for the binary case $r=2$~\cite{sudholt2019choice}, everywhere for a setting of the algorithm
parameter that minimizes the more general runtime bound obtained
in the work.

The generalizations of \ONEMAX, \LO, and \BV
studied in the above works essentially only distinguish between two
different values per component. For example, the generalized
\ONEMAX is defined as $\ronemax(\bm{x}) \coloneqq
    \sum_{i=1}^n \indic{x_i = r-1}$, \ie, the function only counts the
number of components having the optimal value~$r-1$, while all
other values have the same contribution of~$0$.
Similarly, for the multi-valued \LO function investigated with
respect to the \rumda by \citet{BenJedidiaDK24} and with respect to
the \rcga by \citet{AdakWittEvoCOP25}, the function counts only the
leading number of components having a certain value like~$r-1$ and all other values are treated equivalently. Hence, essentially, when combining
the values not contributing to fitness, the problem reduces to
a binary one. This becomes especially evident in the considered
multi-valued EDAs, which maintain a \emph{frequency matrix} $p_{i,j}$,
where $i\in\{1,\dots,n\}$ and $j\in\{0,\dots,r-1\}$, such that component~$i$ takes value~$j$ with probability $p_{i,j}$ in a sampling
step of the algorithm. All entries of the matrix are initialized
at~$1/r$. Crucially, in all above-mentioned analyses, only
the setting of the frequencies $p_{i,r-1}$ of sampling the optimal
value~$r-1$ at the components
needs to be considered, and it does not matter
how the remaining probability mass
$1-p_{i,r-1} = \sum_{j=0}^{r-2} p_{i,j}$ is distributed on the other
$r-1$~values. Hence, as already observed in \citet{HamanoUSMA25}, the
analysis on the multi-valued $\ronemax$ and the multi-valued \rLO essentially proceeds in the same way as for the classic
binary EDAs, with the only difference
that the probability of sampling the optimal value per position starts
at~$1/r$ and not at~$1/2$. Thus, it is debatable whether these functions are truly multi-valued and provide enough insight into the
stochastic behavior of the probabilistic model of a multi-valued
EDA.

Based on these insights, \citet{AdakWittPPSN2024} considered the
multi-valued function $\gonemax\colon R^n \to\R$ defined
by $\gonemax(\bm{x})=\sum_{i=1}^n \bm{x}_i$ and already
analyzed in the context of EAs \citep[\eg,][]{doerr2018static}. In contrast
to \ronemax, this function
truly depends on all $r$ possible values per component while
still having the all $r-1$-string $(r-1)_{i = 1}^n$ as the optimal
solution. They only gave empirical evidence of that the \rcga
has a similar performance on $\gonemax$ as on the quasi-binary
\ronemax. \citet{AdakW25} then in a follow-up work
provided the first runtime bound
for the \rcga on \gonemax of $\bigO{nr^3 \log^2(n)\log(r)}$ for optimized
settings of the algorithm parameter, which constitutes a significantly
higher bound than the $\bigO{nr \log (n)\log (r)}$ obtained for
the \ronemax function.

\textbf{Our contribution.} We revisit the runtime
analysis of the \rcga on the truly multi-valued \gonemax
function and prove with high probability a runtime bound of $\bigO{nr\log^3 (n)\log^3 (r)}$
if $r=\bigO{n^{1/6-\varepsilon}}$ for
an arbitrarily small constant~$\varepsilon>0$ (\Cref{thm:rcga-on-gom}). This
improves the state of the art by a factor of $\bigTheta{r^2/(\log(n)\log^2( r))}$ for the considered range of~$r$.
To achieve these results, we provide new analytical arguments
in two respects:

(1)~A closer
inspection of genetic drift compared to \citet{AdakW25}. More precisely,
they pessimistically ignored that a frequency only changes with
probability at most~$1/r$ right after initialization, \ie, the
frequency has a high probability of staying put. By using
a genetic-drift theorem allowing for such self-loops of the
stochastic process (\Cref{thm:rcgaNeutralConcentration}), adapted from \citet{BenJedidiaDK24}, who used
a similar technique, we save a factor of~$\bigTheta{r}$ in the runtime
bound.

(2)~A fine-grained analysis of the growth of frequencies, considering
all possible $r-1$ values of the spectrum. More precisely, we show
that for each position~$i\in\{1,\dots,n\}$, the probability distribution
$p_{i,\cdot}$ on the $r-1$ different
values quickly concentrates on smaller and smaller
intervals of $\{0,\dots,r-1\}$. While it is uniform to begin with,
after time $\bigO{nr\log^3(n)\log^2 (r)}$ with high probability, roughly speaking, all mass has moved
to an upper interval $\{r/2,\dots,r-1\}$ of half the original size,
and the probability mass~$p_{i,j}$ is essentially monotonically
increasing in~$j$, as
our intuition suggests, since sampling a higher value gives a higher fitness. Iterating this halving argument $\bigO{\log r}$ times gives
the final runtime bound.

To make this very rough idea formal, several obstacles have to
be overcome. First of all, genetic drift may lead to minor distortions
of the monotonicity hoped for. Here we again use the strong genetic
drift theorems mentioned above to limit the error.
More severely, even if a position
is currently not exposed to genetic drift, the algorithm may be unlucky
and more often than expected sample relatively small values within the target interval (like $\{r/2,\dots,r-1\}$)
that the mass is moved into. This has an undesired self-reinforcing effect
since if a value is sampled and its frequency is updated, the following
steps give more preference to that value, and the monotonicity
of the frequencies $p_{i,j}$ on the values does not hold any longer. However,
for a sufficiently large choice of the update parameter~$K$ of the
\rcga, averaging effects limit this undesired reinforcement. We make
this formal by deriving a new Chernoff bound for processes with
time-dependent and thereby no longer independent success probabilities (\Cref{thm:success-upper-parts}).

Further technical challenges are solved by decomposing the process
into genetic-drift steps and steps where a frequency experiences
a fitness signal (so-called biased steps), and a
refined analysis of the probability
of a biased step. Unlike previous analyses, the refined analysis
allows us to classify
more steps as biased, including situations
where the search points of the \rcga do not sample exactly the same
fitness in $n-1$ positions. Here a useful and novel auxiliary lemma (\Cref{lem:probabilityShiftingLowerBound})
dealing with the convolution of probability functions comes into play.
Taken together, these tools avoid the pessimistic assumption
of a sampling variance of $\bigTheta{nr}$ that \citet{AdakW25} made
in their initial analysis. Here we save the second $\bigTheta{r}$
factor in the runtime bound (only incurring
new factors being polylogarithmic in~$r$).

\textbf{Limitation.}
Our current proof relies heavily on the property that the relative difference of probability mass among well-chosen blocks of values remains almost the same throughout the run (\Cref{lem:adequate-biased-growth}).
We prove this property via our concentration result for random processes with changing success probabilities (\Cref{thm:success-upper-parts}).
However, since some of the blocks are small and thus initially only in the order of~$\frac{1}{r}$, we only obtain a desired concentration for values of $r = \bigO{n^{1/6-\varepsilon}}$ for an arbitrarily small constant~$\varepsilon > 0$.

\textbf{Outline.}
We introduce our notation, setting, and mathematical tools in \Cref{sec:preliminaries}.
In \Cref{sec:main-result}, we present our main result (\Cref{thm:rcga-on-gom}) and discuss the outline of its proof in different levels of details, introducing various notation.
In \Cref{sec:building-blocks}, we present the ideas discussed in \Cref{sec:main-result} in rigorous detail and conclude with a proof sketch of our main result.
We conclude our paper in \Cref{sec:conclusion}.

All proofs are found in the appendix (\Cref{sec:appendix}).

\section{Preliminaries}
\label{sec:preliminaries}

We first define the setting we consider, including the $\gonemax$ fitness function that we study.
Afterward, we present the \rcga, which is the algorithm that we analyze.
Last, we present several mathematical tools that we use in our analysis.

Throughout this article, we let~$\N$ denote the set of all natural numbers, including~$0$, and we let~$\R$ denote the set of all reals.
For all $a, b \in \R$, we let $[a .. b] \coloneqq [a, b] \cap \N$ as well as $[b] \coloneqq [1 .. b]$.
Moreover, given a random process $(X_t)_{t \in \N}$, we refer to the process simply as~$X$ for short.

\subsection{The $r$-Valued Optimization Setting}

Let $n, r \in \N_{\geq 2}$.
We consider the maximization of $r$-valued \emph{fitness} functions, \ie, of functions $f\colon [0 .. r - 1]^n \to \R$.
We call each $\bm{x} \in [0 .. r - 1]^n$ an \emph{individual}, and we call~$f(\bm{x})$ the \emph{fitness of~$\bm{x}$}.
For each $i \in [n]$, we denote component~$i$ of~$\bm{x}$ by~$\bm{x}_i$.

We analyze the $\gonemax$ (\globalOneMax) fitness function~\cite{AdakWittPPSN2024}, which is the sum of all the values in the individuals.
Formally, we have $\globalOneMax\colon \bm{x} \mapsto \textstyle\sum\nolimits_{i \in [n]} \bm{x}_i$.

\subsection{The $r$-Valued Compact Genetic Algorithm}

The \emph{$r$-valued compact genetic algorithm} (\rcga) was introduced by \citet{BenJedidiaDK24} as an extension of the \emph{compact genetic algorithm} (cGA) by \citet{HarikLG99cGA}, for the binary domain.
It aims at maximizing an $r$-valued fitness function of dimension $n \in \N_{\geq 2}$.

\begin{algorithm}[t]
    \caption{$r$-valued Compact Genetic Algorithm (\rcga) for the maximization of $f\colon [0 .. r-1]^n \rightarrow \R$}
    \label{algorithm:r-cGA-rOneMax}
    $t \gets 0$\;
    \lFor{$(i,j) \in [n] \times [0 .. r-1]$}{%
        $p^{(t)}_{i,j} \gets \frac{1}{r}$%
    }
    \While{termination criterion not met}{
        $\sample{t}{1}{i}, \sample{t}{2}{} \gets$ two individuals independently sampled from~$\freq{t}{}{}$\;
        $\selectedSample{t}{1}{} \gets$ individual with the higher fitness from $\set{\sample{t}{1}{}, \sample{t}{2}{}}$ (breaking ties uniformly at random)\;
        $\selectedSample{t}{2}{} \gets$ individual from $\set{\sample{t}{1}{}, \sample{t}{2}{}} \smallsetminus \set{\selectedSample{t}{1}{}}$\;
        \For{$i\in [n]$}{
            \For{$j\in [0 .. r-1]$}{
                $\freq{t + 1}{i}{j} \gets \freq{t}{i}{j} + \frac{1}{K} (\indic{\selectedSample{t}{1}{i} = j} - \indic{\selectedSample{t}{2}{i} = j})$\;
            }
        }
        $t\gets t + 1$\;
    }
\end{algorithm}

The \rcga (\Cref{algorithm:r-cGA-rOneMax}) maintains a probability matrix $(\freq{}{i}{j})_{i \in [n], j \in [0 .. r - 1]}$ (the \emph{frequency matrix}), where each row sums to~$1$ and represents a probability distribution over $[0 .. r - 1]$.
Initially, each row represents a uniform distribution.
In each iteration, each row is used independently in order to create two samples such that for each $i \in [n]$ and each $j \in [0 .. r - 1]$, \emph{frequency}~$\freq{}{i}{j}$ denotes the probability to create value~$j$ at position~$i$.
Formally, for all random samples $\bm{y} \in [0 .. r - 1]^n$ and all (deterministic) individuals $\bm{x} \in [0 .. r - 1]^n$, defining $0^0 \coloneqq 1$, we have
\begin{equation*}
    \prob{\bm{y} = \bm{x} \given \freq{}{}{}}
    = \textstyle\prod\nolimits_{i \in [n]} \textstyle\prod\nolimits_{j \in [0 .. r - 1]} (\freq{}{i}{j})^{\indic{\bm{y}_i = j}} .
\end{equation*}

Given two samples, the \rcga updates each row $i \in [n]$ of the frequency matrix by increasing the frequency of the value at position~$i$ of the better sample, and reducing the frequency of the value of the worse sample by the same amount.
This amount is determined by the algorithm parameter $K \in \N_{\geq 1}$, known as the \emph{hypothetical population size}.
In order to not introduce mathematical border cases, we augment the terminology introduced by \citet{Doerr21cGAonJump} and say that~$K$ is \emph{well-behaved} if and only if~$K$ is divisible by~$r$.
This ensures that updates of~$\frac{1}{K}$ to frequencies may exactly reach the extreme values~$0$ and~$1$, and are not cut off.

We note that the \rcga by \citet{BenJedidiaDK24} is more general and allows to restrict the frequency values to a closed subset of $[0, 1]$ in order to prevent the frequencies from reaching the extreme values~$0$ and~$1$.
We do not consider such non-trivial borders, as they make the already challenging analysis even more challenging.
However, we believe that our arguments carry over to the setting of non-trivial borders, although with sufficiently additional effort.

\subsection{Mathematical Tools}

We make use of the following theorems, each dealing with a different aspect.

The first theorem provides us with strong concentration bounds for martingales, based on the variance of the process.
Since the frequencies of the \rcga all start at~$\frac{1}{r}$, their initial variance is small.
Hence, such a theorem provides us with stronger bounds than ones that do not incorporate the variance.

\begin{theorem}[{\cite[Theorem~$3.15$]{McDiarmid98}}]
    \label{thm:martingaleConcentration}
    Let $(X_{t})_{t \in \N}$ be a bounded martingale with respect to a filtration $(\mathcal{F}_t)_{t \in \N}$.
    Further, for all $t \in \N_{\geq 1}$,  denote the deviation by $\mathrm{dev}_{t} \coloneqq |X_t - X_{t - 1}|$.
    In addition, let $b=\sup_{t \in \N} \mathrm{dev}_{t}$, and assume that~$b$ is finite.
    Last, for all $t \in \N$, let $\hat{v}_t = \sup \sum_{s \in [t]} \Var[X_s - X_{s -1} \mid \mathcal{F}_{s - 1}]$.
    Then for all $t \in \N$ and all $\varepsilon \in \R_{\geq 0}$, we have
    \begin{equation*}
        \prob*{\exists s \in [0 .. t]\colon \abs*{X_t - \expect{X_0}} \geq \varepsilon} \leq 2 \exp\left(-\frac{\varepsilon^{2}}{2\hat{v}_t + 2b\varepsilon/3}\right).
    \end{equation*}
\end{theorem}

The next theorem is the runtime result by \citet{AdakW25}, which we use in order to cover some easy cases in our analysis.

\begin{theorem}[{\cite[Theorem~$4.6$]{AdakW25}}]
    \label{thm:rcga-runtime-adak-witt}
    Let $r \in \N_{\geq 2}$ with $r = \poly(n)$, and let $\cstar \in \R_{> 0}$ be a sufficiently large constant.
    Consider the \rcga optimizing \globalOneMax, with a well-behaved $K \geq \cstar r^2 \sqrt{n} \ln n$ with $K = \poly(n)$.
    Then with probability at least $1 - \bigO{\frac{1}{n}}$, the \rcga finds the global optimum after $\bigO[\big]{K r \sqrt{n}(\log r + \log K)}$ iterations and function evaluations.
\end{theorem}

The last theorem follows well known concentration results for the multiplicative drift theorem~\cite[\eg,][]{DoerrG10MultiplicativeDrift,LenglerS18DriftRevisited}.
Since we only have bounds on the steps of our process for a finite time horizon, we adapt these results to the setting that incorporates such stochasticity.
We note that \cref{eq:multiplicative-drift-concentration-with-failure:drift-bound} in the following theorem essentially is a conditional drift condition, very similar to the drift theorems developed by \citet{HamanoUSMA25}.

\begin{restatable}{theorem}{multiDriftWithFailure}
    \label{thm:multiplicative-drift-concentration-with-failure}
    Let $(X_t)_{t \in \N}$ be an integrable random process over~$\R_{\geq 0}$, adapted to a filtration $(\filt{t})_{t \in \N}$.
    Let $\xmin \in \R_{> 0}$ and $T = \inf\set{t \in \N \given X_t < \xmin}$.
    Furthermore, let $t' \in \N \cup \set{\infty}$ and $q \in (0, 1]$, and let $(A_t)_{t \in [0 .. t']}$ be a collection of events\footnote{In case that $t' = \infty$, we interpret the set $[0 .. t']$ as~$\N$.}.
    Assume that there is a $\delta \in (0, 1)$ such that for all $t \in [0 .. t']$, we have
    \begin{equation}
        \label{eq:multiplicative-drift-concentration-with-failure:drift-bound}
        \expect{(X_t - X_{t + 1}) \cdot \indic{t < T \land \cap_{s \in [0 .. t']} A_s} \given \filt{t}}
        \geq \delta X_t \cdot \indic{t < T} \prob{\cap_{s \in [0 .. t']} A_s \given \filt{t}} .
    \end{equation}

    Then we have for all $t \in [0 .. t']$ that
    \begin{align*}
        \prob{T > t}
         & \leq \prob{X_t \cdot \indic{t < T} \geq \xmin \land \cup_{s \in [0 .. t']} \overline{A_s}}                               \\
         & \quad+ (1 - \delta)^t \cdot \frac{\expect{X_0 \cdot \indic{T > 0 \land \cap_{s \in [0 .. t']} A_s}}}{\xmin}              \\
         & \leq \prob{\cup_{s \in [0 .. t']} \overline{A_s}} + (1 - \delta)^t \cdot \frac{\expect{X_0 \cdot \indic{T > 0}}}{\xmin}.
    \end{align*}
    In particular, if $T > 0$, and there is a $q \in [0, 1]$ such that $\prob{\cup_{s \in [0 .. t']} \overline{A_s}} \leq q$, then for all $\gamma \in \R_{> 0}$, if $t' \geq (\ln(\expect{X_0} / \xmin) + \gamma) / \delta \eqqcolon \beta$, then $\prob{T > \beta} \leq q + \eulerE[-\gamma]$.
\end{restatable}

\section{Main Result and Proof Overview}
\label{sec:main-result}

Our main result bounds the runtime of the \rcga on the \globalOneMax fitness function with high probability while only depending linearly on~$r$, up to polylogarithmic factors.
As mentioned above, it requires, for any constant $\varepsilon \in \R_{> 0}$, that $r = \bigO{n^{1 / 6 - \varepsilon}}$, as well as a sufficiently large~$K$.
For the best choice of~$K$, it improves the previously best bound by \citet{AdakW25} (\Cref{thm:rcga-runtime-adak-witt}) by a factor of order $r^2 / \bigl(\log(n) \log^2(r)\bigr)$, which is asymptotically larger than~$1$ for $r = \bigOmega{\log n}$.

\begin{restatable}{theorem}{rcgaOnGom}
    \label{thm:rcga-on-gom}
    Let $r \in \N_{\geq 2}$ with $r = \bigO{n^{1 / 6 - \varepsilon}}$ for a constant $\varepsilon \in \R_{> 0}$, and let $\cstar \in \R_{> 0}$ be a sufficiently large constant.
    Consider the \rcga optimizing \globalOneMax, with a well-behaved $K \geq \cstar r \sqrt{n} \ln^2(n)\ln^2(r)$ with $K = \poly(n)$.
    Then with probability at least $1 - \bigO{\frac{1}{n}}$, the \rcga finds the global optimum after $\bigO[\big]{K \sqrt{n} \log(n) \log(r)}$ iterations and function evaluations.

    For the lowest admissible value of~$K$, this results in a runtime bound of $\bigO[\big]{n r \log^3(n) \log^3(r)}$ iterations and function evaluations.
\end{restatable}

As explained in the introduction, we save a factor of~$r^2$ in comparison to the result by \citet{AdakW25} by saving a factor of~$r$ in two occasions:
(1) We conduct a tighter analysis for how frequencies behave in case of only receiving random noise as input for their update (known as \emph{genetic drift}).
(2) We conduct a tighter analysis for how likely it is that frequencies do receive a non-noisy input (known as \emph{biased steps}).

\textbf{High-level proof idea.}
The main idea of the proof of \Cref{thm:rcga-on-gom} applies the ideas of the proof technique for the binary case (as done by \citet{sudholt2019choice}) in an iterative fashion.
Roughly speaking, for each position $i \in [n]$, we create blocks of consecutive frequencies of initially approximately the same probability mass that take on the roles of the two values~$0$ and~$1$ in the binary case.
Once all probability mass from the block resembling~$0$ is removed, we apply our argument recursively to the remaining frequencies.
We call the time that it takes for this shift of probability mass a \emph{phase}.
Since we show that the block resembling the value~$0$ initially represents a constant fraction of the remaining frequencies with positive mass, we require overall roughly~$\ln r$ phases.

In order to combine our recursive argument for all~$n$ positions, we utilize that the process during each phase (per position) exhibits multiplicative drift, which is strongly concentrated around its expectation.
Hence, by relying on \Cref{thm:multiplicative-drift-concentration-with-failure}, we obtain that each phase for each position lasts with high probability not longer than order $K \sqrt{n} \ln n$.
A union bound over all~$n$ positions yields that a single phase for all~$n$ positions also only lasts order $K \sqrt{n} \ln n$ with high probability.
Iterating this argument for all phases, as discussed above, yields the final result.

\textbf{More detailed proof idea.}
In order to make our ideas more formal, we first formalize the concept of \emph{random-walk steps} and \emph{biased steps}, which dates back to the analysis by \citet{sudholt2019choice} for the binary case.
In a nutshell, when considering a single position $i \in [n]$, the update performed receives a meaningful input if the difference in fitness of both samples, excluding position~$i$, is smaller than the difference between the higher value and the lower value sampled at position~$i$.
In this case, the value at position~$i$ determines the order of the two samples, and the algorithm learns that higher values of \globalOneMax are better than lower ones.
In the other cases, the value at position~$i$ does not determine the ranking of the two samples, and any value at position~$i$ may receive an update (in particular, a higher value may be ranked worse than a lower value).

An update, and consequently an iteration, where position~$i$ determines the ranking of the two samples is called a \emph{biased step}.
All other updates are called \emph{random-walk steps}.

More formally, let $t \in \N$ be an iteration of the \rcga optimizing \globalOneMax, and let $i \in [n]$ be a row index of the frequency matrix.
Using the notation from \Cref{algorithm:r-cGA-rOneMax}, we define
\begin{align}
    \label{eq:sumOfSamples}
    \sumOfSampleValues{t}{1}{i}
     & = \textstyle\sum\nolimits_{k \in [n] \smallsetminus \set{i}} \sample{t}{1}{k}
    \quad \textrm{and}                                                               \\
    \notag
    \sumOfSampleValues{t}{2}{i}
     & = \textstyle\sum\nolimits_{k \in [n] \smallsetminus \set{i}} \sample{t}{2}{k}
    \quad \textrm{as well as}                                                        \\
    \notag
    D^{(t)}_i
     & = \sumOfSampleValues{t}{1}{i} - \sumOfSampleValues{t}{2}{i} .
\end{align}
We have a biased step in iteration~$t$ if $\max\set{\sample{t}{1}{i}, \sample{t}{2}{i}} - \min\set{\sample{t}{1}{i}, \sample{t}{2}{i}} > D^{(t)}_i$.
In all other cases, we have random-walk steps.

In order to carefully follow the progress of the \rcga, we group the frequencies at each position $i \in [n]$ into different blocks so that it is likely to observe a frequency increase in the block of the highest frequencies.
Roughly speaking, our grouping splits the frequencies into four blocks:
(1) The lowest-value frequencies that are already at~$0$ and do not contribute anything to the algorithm anymore since sampling from the corresponding values has become impossible.
(2) The next-higher frequencies whose maximum value is chosen as explained below in \cref{eq:intervalHierarchy}.
(3) Some intermediate-value frequencies whose maximum value is also chosen as explained below.
(4) And the highest-value frequencies that have at least an overall constant probability mass.

The idea behind this grouping is that the difference in values of blocks~(4) and~(2) are so large that it is likely to be higher than the difference~$D_i$ in fitness of the two samples in all other positions.
Vaguely, we separate blocks~(4) and~(2) by so much that their difference is at least as large as the standard deviation of the sampling variance, call it~$\sigma$, in the fitness difference~$D_i$.
We show that, similarly to the binary case, the probability to have a sufficiently small difference~$D_i$ is in the order of~$\sigma^{-1}$ (\Cref{lem:biasedStepLowerBound}).
Given our grouping,~$\sigma$ is in the order of~$\sqrt{n}$, resulting in a multiplicative drift with a factor of order $(K \sqrt{n})^{-1}$ at position~$i$ (\Cref{lem:driftBound}).

We note that we prove \Cref{lem:biasedStepLowerBound} via a lemma that carefully studies the convolution of rather arbitrary functions and may thus be of independent interest (\Cref{lem:probabilityShiftingLowerBound}).

Formally, for our grouping, let $\kappa^* = \ceil{\log_{3/2}(r - 1)}$, noting for $r > 2$ that $\kappa^* > \log_{3/2}(r - 1)$, because any positive-integer power of~$\frac{3}{2}$ is not an integer, as~$2$ and~$3$ are both prime.
We define for all $\kappa \in [0 .. \kappa^* - 1]$ the index
\begin{align}
    \label{eq:intervalHierarchy}
    \ell_{\kappa} & = \ceil*{\left(1 - \left(\frac{2}{3}\right)^\kappa\right)(r - 1)} \textrm{ as well as the intervals} \\
    \notag
    K_\kappa
                  & = \left[\ell_{\kappa}  ..
        \ell_{\kappa+1} - 1 \right] \textrm{ with }
    K_{\kappa^*} = \set{r - 1} \textrm{ and }                                                                            \\
    \notag
    S_\kappa      & = \textstyle\bigcup_{\nu \in [\kappa .. \kappa^*]} K_\nu .
\end{align}

Note that $(K_\kappa)_{\kappa \in [0 .. \kappa^*]}$ is a partition of $[0 .. r - 1]$, where it is important to check that $(1 - (\frac{2}{3})^{\kappa^*})(r - 1) > r - 2$.
Hence, $K_{\kappa^* - 1}$ stops at $r - 2$, which is exactly by~$1$ smaller than the start of~$K_{\kappa^*}$.

In order to refer to the probability mass to the intervals defined in \cref{eq:intervalHierarchy}, we define for all $I \subseteq [0 .. r - 1]$ and all $i \in [n]$ the notation
\begin{equation}
    \label{eq:probabilityMassNotation}
    \mu_i^{(t)}(I) \coloneqq \textstyle\sum\nolimits_{j\in I} \freq{t}{i}{j} .
\end{equation}

Given this terminology and notation, we prove a concentration result for random-walk steps (\Cref{thm:rcgaNeutralConcentration}), which shows that the overall probability mass that we lose is overall negligible (\Cref{lem:contribution-random-walk-steps-kappastar}).

Afterward, we show that the overall fraction of neighboring upper intervals from \cref{eq:intervalHierarchy} remains per phase roughly the same, except for the blocks whose probability mass is lowest and shifted to higher blocks (\Cref{lem:adequate-biased-growth}).
This guarantees that we have roughly the same starting setup for each phase.

We prove \Cref{lem:adequate-biased-growth} via an interesting theorem that is similar to classic Chernoff bounds but allows for a change in the success probability (\Cref{thm:success-upper-parts}), which reflects the update in probability mass that the intervals receive.

In the end, our main result is essentially a combination of our drift bound (\Cref{lem:driftBound}) and our result that each phase behaves roughly the same, only with the probability mass shifting to higher blocks (\Cref{lem:adequate-biased-growth}).

\section{The Building Blocks of Our Main Proof}
\label{sec:building-blocks}

We provide the mathematically rigorous details to the concepts discussed in \Cref{sec:main-result}.
In the statements, we refer to the definitions that we gave above, but we also introduce specific notation in particular cases that is mostly only needed for local instances.

\subsection{Probability of Biased Steps}
The \rcga makes progress to the optimum
thanks to the above-mentioned
biased steps, which implies a drift
of (sums of) frequencies. We bound the
probability of an event required
for a biased step in
\Cref{lem:biasedStepLowerBound} below.
Its proof is crucially based on the
following, more general auxiliary statement dealing
with a convolution.

\begin{restatable}{lemma}{probabilityShiftingLowerBound}
    \label{lem:probabilityShiftingLowerBound}
    Let $\rho \in \N$, $m \in \N_{\geq 1}$, $\rho^+ = \max \set{1, \rho}$, and $q\colon [m \rho^+] \to \R_{\geq 0}$ be arbitrary.
    Moreover, let $M \in \R_{\geq 0}$ be such that $\sum_{i \in [m \rho^+]} q(i) = M$.
    Then
    \begin{equation}
        \label{eq:probabilityShiftingLowerBound}
        \textstyle\sum\nolimits_{d \in [0 .. \rho]} \textstyle\sum\nolimits_{i \in [m \rho^+ - d]} q(i) q(i + d)
        \geq \frac{M^2}{2 m}.
    \end{equation}
\end{restatable}

The following lemma bounds the probability of certain
outcomes of the difference in fitness for all positions but position $i \in [n]$ (\cref{eq:sumOfSamples}).

\begin{restatable}{lemma}{biasedStepLowerBound}
    \label{lem:biasedStepLowerBound}
    Let $r, n \in \N_{\geq 2}$.
    Consider the \rcga optimizing \globalOneMax, with arbitrary well-behaved~$K$.
    Let $i \in [n]$, let $\Delta \in [0 .. (n - 1) (r - 1)]$, and let $t' \in \N$ denote a random iteration.
    Last, using the notation from \cref{eq:sumOfSamples}, let $\sigma \in \R_{\geq (\Delta + 2) / 4}$ such that for all $\ell \in \set{1, 2}$, we have $\var{\sumOfSampleValues{t'}{\ell}{i} \given \freq{t'}{}{}} \leq \sigma^2$ as well as $\floor{\expect{\sumOfSampleValues{t'}{\ell}{i} \given \freq{t'}{}{}} - 2 \sigma} + 1 \geq 0$ and $\ceil{\expect{\sumOfSampleValues{t'}{\ell}{i} \given \freq{t'}{}{}} + 2 \sigma} - 1 \leq (n - 1) (r - 1)$.
    Then
    \begin{equation*}
        \prob[\big]{\differenceOfSampleValues{t'}{i} \in [0 .. \Delta]}
        \geq \frac{9 \max \set{1, \Delta}}{32 (4 \sigma - 1)} .
    \end{equation*}
\end{restatable}

The following lemma is an auxiliary
result that bounds
the variance of the sampling distribution,
which typically enters our bounds
on probabilities of biased steps (\eg,
\Cref{lem:biasedStepLowerBound} above)
and drift bounds. The lemma allows
that certain subintervals of frequency
values are empty, which, as explained above, is at the
core of our main proof strategy.

\begin{restatable}{lemma}{varianceUpperBound}
    \label{lem:varianceUpperBound}
    Recall the notation from \cref{eq:sumOfSamples}.
    Let $r \in \N_{\geq 2}$ and $n \in \N_{\geq 2}$.
    Consider the \rcga optimizing \globalOneMax, with arbitrary well-behaved~$K$.
    Let $i \in [n]$ and let $\jStar \in [0 .. r - 1]$.
    Furthermore, let $t' \in \N$ be a random iteration such that for each $k \in [n] \smallsetminus \set{i}$ the sum of all frequencies at position~$k$ with values in $[\jStar .. r - 1]$ is~$1$.
    Then we have for $\ell \in \set{1, 2}$ that
    \begin{equation*}
        \var{\sumOfSampleValues{t'}{\ell}{i} \given \freq{t'}{}{}}
        \leq (n - 1) (r - 1 - \jStar)^2  .
    \end{equation*}
\end{restatable}

\textbf{Large Biased Steps and Drift Bound.}
We precisely define
the intervals of frequency values
between which probability mass is moved, and we
introduce a more specific type
of large biased steps (\Cref{def:large-biased-step}). With these tools at hand, we
prove bounds on the drift of frequencies
between subintervals in tandem with
probabilities of large biased steps (\Cref{lem:driftBound}).
This
statement is applied in different phases and
subintervals in \Cref{subsect:bounds-phases}.

\begin{definition}
    \label{def:large-biased-step}
    Recall the notation from \cref{eq:intervalHierarchy}.
    Given $i \in [n]$, $t \in \N$, and $\kappa \in [0 .. \kappa^* - 2]$ the event $B_{i,\kappa,t}$, called a \emph{large biased step}, occurs at time~$t$ in position~$i$, if and only if one sample in iteration~$t$ has
    a value in $K_\kappa$ and the other one a value in $S_{\kappa+2}$ and the step is biased, \ie, the frequency update is performed with respect
    to the sample with its value in $S_{\kappa+2}$.
\end{definition}

\begin{restatable}{lemma}{driftBound}
    \label{lem:driftBound}
    Recall the notation from \cref{eq:sumOfSamples,eq:intervalHierarchy,eq:probabilityMassNotation} and \Cref{def:large-biased-step}.
    Let $r \in \N_{\geq 10}$ and $n \in \N_{\geq 4}$.
    Consider the \rcga optimizing \globalOneMax, with arbitrary well-behaved~$K$.
    Let $i \in [n]$, let $t'$ be a random iteration, and assume
    $\mu_i^{t'}(S_\kappa)=1$ for some $\kappa\ge 0$ satisfying $\ell_\kappa\le r-10$.   Assume furthermore that there is a constant~$\cdrift>0$ (not depending on~$i$ and~$t'$) such that $\mu_i^{(t')}(S_{\kappa+2}) \ge \cdrift \cdot \mu_i^{(t')}(S_{\kappa+1})$.
    Moreover,
    let
    \begin{align*}
        D_i^{(t')}(\kappa) & \coloneqq
        \frac{9\cdrift\cdot \sDstar (1-\sDstar) \max \set{1, (r-1-\ell_\kappa)/2}}{32 \bigl(4
            \sqrt{(n - 1) (r - 1 - \ell_\kappa)^2 }-1\bigr)}
        \\ & \ge
        \frac{\cdrift \cdot\sDstar (1-\sDstar)}{29\sqrt{n}},
    \end{align*}
    where the last estimate holds for $\ell_\kappa\le r-2$ and sufficiently
    large~$n$.
    Then we have for
    all $t\ge t'$ that
    \begin{align*}
         & \prob{B_{i,\kappa,t'}}  \ge  D_i^{(t')}(\kappa) \quad \textrm{and} \\
         & \expect{\sTstar{t'+1} - \sTstar{t'}\given \freq{t'}{}{}}
        \geq \frac{D_i^{(t')}(\kappa) }{K} .
    \end{align*}
    Moreover, the same bound holds for $\expect{\sDstarP{t'+1} - \sDstarP{t
                '}\given \freq{t'}{}{}}$.
\end{restatable}

\subsection{Bounds for the Change
    of Frequency Sums in Phases}
\label{subsect:bounds-phases}

Recall the notation from \cref{eq:intervalHierarchy}.
During each phase $\kappa \in [0 .. \kappa^* - 1]$, we show that for each $\kappa' \in [\kappa + 1 .. \kappa^*]$, each sum of frequencies in~$S_{\kappa'}$ grows to about the fraction of~$S_{\kappa'}$ in comparison to~$S_{\kappa + 1}$, removing all probability mass from~$K_\kappa$ with high probability.
To this end, we study the contribution due to random-walk steps and due to biased steps separately, defining complementary sequences of increasing stopping times.
The foundation are the events, similar to those defined in \Cref{def:large-biased-step}, that define for each point in time whether it is a biased step or not.
Formally, for all $i \in [n]$ and all $t \in \N_{\geq 1}$, recalling the notation of \Cref{algorithm:r-cGA-rOneMax} as well as of \cref{eq:sumOfSamples}, we define the event for a random-walk step as
\begin{equation}
    \label{eq:random-walk-event}
    \mathcal{E}^{(t)}_i = \indic{\selectedSample{t - 1}{1}{i} - \selectedSample{t - 1}{2}{i} \leq D^{(t - 1)}_i} ,
\end{equation}
noting that~$\mathcal{E}^{(t)}_i$ is measurable with respect to~$\freq{t}{}{}$.

Let $t' \in \N$ be a random iteration, and let $i \in [n]$.
Based on \cref{eq:random-walk-event}, we define the sequence of increasing stopping times for the random-walk steps at position~$i$, starting from~$t'$, for all $s \in \N_{\geq 1}$ inductively as
\begin{align}
    \label{eq:random-walk-stopping-times}
    \randomWalkStoppingTime{i}{t'}[0] = t' \quad \textrm{and} \quad
    \randomWalkStoppingTime{i}{t'}[s] = \inf \set{t \in \N \mid t > \randomWalkStoppingTime{i}{t'}[s - 1] \land \mathcal{E}^{(t)}_i \textrm{ occurs}} .
\end{align}
Similarly, for the biased steps, for all $s \in \N_{\geq 1}$, we inductively define the sequence
\begin{align}
    \label{eq:biased-stopping-times}
    \biasedStoppingTime{i}{t'}[0] = t' \quad \textrm{and} \quad
    \biasedStoppingTime{i}{t'}[s] & = \inf \set{t \in \N \mid t > \biasedStoppingTime{i}{t'}[s - 1] \land \mathcal{E}^{(t)}_i \textrm{ does not occur}} .
\end{align}
Note that both sequences are strictly increasing by definition and define stopping times each with respect to the natural filtration of $(\freq{t' + t}{}{})_{t \in \N}$.

For a position $i \in [n]$ and an index $\kappa \in [\kappa^*]$ for an upper frequency sum, we define the \emph{contribution} of the probability mass in~$S_\kappa$ in terms of random-walk and biased steps, starting from a random iteration $t' \in \N$.
That is, for an increasing sequence of stopping times $(U_t)_{t \in \N}$, such as \cref{eq:random-walk-stopping-times} or~\eqref{eq:biased-stopping-times}, we define
\begin{equation}
    \label{eq:contribution-delta}
    \bigl(\Delta^{(U_t)}_i(S_{\kappa})\bigr)_{t \in \N_{\geq 1}} = \bigl(\mu^{(U_t)}_i(S_\kappa) - \mu^{(U_t - 1)}_i(S_\kappa)\bigr)_{t \in \N_{\geq 1}} ,
\end{equation}
noting that the time point $U_t - 1$ does not need to be an element from~$U$.
Last, using \cref{eq:contribution-delta}, we define the \emph{change} of the frequency mass of~$S_\kappa$ as
\begin{equation}
    \label{eq:filtered-change}
    \bigl(\mu^{(t)}_{i \mid U}(S_\kappa)\bigr)_{t \in \N} = \bigl(\mu^{(U_0)}_i(S_\kappa) + \sum\nolimits_{s \in [t]} \Delta^{(U_s)}_i(S_{\kappa})\bigr)_{t \in \N} ,
\end{equation}
where we choose for~$U$ again either~$\randomWalkStoppingTime{i}{t'}[]$ or~$\biasedStoppingTime{i}{t'}[]$.

Next, we analyze
the change in random-walk and biased
steps more closely.

\textbf{Bounds on Frequency Decrease due to
    Genetic Drift.}
We now formulate two statements that we use to bound unwanted effects of
genetic drift. The first one  (\Cref{thm:rcgaNeutralConcentration})
applies \Cref{thm:martingaleConcentration} to
intervals of frequencies of the
\mbox{\rcga} under a martingale assumption, \ie,
for random-walk steps. The second statement (\Cref{lem:contribution-random-walk-steps-kappastar}) is a concrete
application of \Cref{thm:rcgaNeutralConcentration} in the setting of our main theorem, considering the contribution of
random-walk steps.

\begin{restatable}{theorem}{rcgaNeutralConcentration}
    \label{thm:rcgaNeutralConcentration}
    Let $r \in \N_{\geq 2}$, let~$f$ be an $r$-valued fitness function of dimension $n \in \N_{\geq 1}$.
    Consider the \rcga optimizing~$f$, with an arbitrary well-behaved~$K$.
    Let $t' \in \N$ be a random iteration.
    Let $i \in [n]$, let $J \subseteq [0 .. r - 1]$ with $J \neq \emptyset$, and for all $t \in \N$, let $P_t = \sum_{j \in J} \freq{t' + t}{i}{j}$.
    Furthermore, let $(U_t)_{t \in \N}$ be a sequence of increasing stopping times, each with respect to the natural filtration of $(\freq{t' + t}{}{})_{t \in \N}$, and assume that $(P_{U_t})_{t \in \N}$ is a martingale.
    Last, let $\alpha \in \R_{\geq 0}$ and $\beta = \max \set[\big]{\min \set{(1 - \alpha) P_{U_0}, 1 - (1 - \alpha) P_{U_0}}, \min \set{(1 + \alpha) P_{U_0}, 1 - (1 + \alpha) P_{U_0}}}$.
    Then, for all $t \in \N$, we have
    \begin{equation*}
        \prob[\big]{\exists s \in [0 .. t]\colon \abs*{P_{U_s} - P_{U_0}} \geq \alpha P_{U_0} \given P_{U_0}}
        \leq 2 \exp\left(-\frac{3 (\alpha P_{U_0} K)^2}{4 \cdot \max \{6 t \beta, \alpha P_{U_0} K\}}\right) .
    \end{equation*}
\end{restatable}

\begin{restatable}{lemma}{contributionRandomWalkStepsKappastar}
    \label{lem:contribution-random-walk-steps-kappastar}
    Recall the notation from \cref{eq:probabilityMassNotation,eq:intervalHierarchy,eq:random-walk-stopping-times,eq:filtered-change}.
    Let $r \in \N_{\geq 3}$, and let $\cstar, \cstop \in \R_{> 0}$ be sufficiently large constants.
    Consider the \rcga optimizing \globalOneMax, with a well-behaved $K \geq \cstar r \sqrt{n} \ln^2(n) \ln^2(r)$.
    Let $i \in [n]$, let $t' \in \N$ be a random iteration, let $\nu \in [0 .. \kappa^*]$, where $\kappa^*\ge 2$, and assume that $\mu^{(t')}_i(S_{\nu}) \geq \frac{\eulerE[-3]}{r}$.
    Then, abbreviating $\tstop = \cstop K  \sqrt{n} \ln n$, we have
    \begin{align*}
         & \prob*{
            \exists s \in [0 .. \tstop]\colon
            \abs*{\mu_{i \mid \randomWalkStoppingTime{i}{t'}[]}^{(t' + s)}(S_{\nu}) - \mu_{i \mid \randomWalkStoppingTime{i}{t'}[]}^{(t')}(S_{\nu})}
            \geq \frac{1}{\kappa^*} \mu_{i \mid \randomWalkStoppingTime{i}{t'}[]}^{(t')}(S_{\nu})
            \given \mu_{i \mid \randomWalkStoppingTime{i}{t'}[]}^{(t')}(S_{\nu})
        }                                          \\
         & \quad \leq 2 n^{-\cstar /(3857\cstop)}.
    \end{align*}
\end{restatable}

\textbf{Bounds on Increases
    due to Biased Steps.}
Next, we state the
central \Cref{lem:adequate-biased-growth}. Roughly speaking, using the notation from \cref{eq:intervalHierarchy}, it shows
that frequency mass in intervals
$S_\nu$ and $S_{\nu+1}$ stays roughly
in the same ratio throughout a phase
that moves probability mass into $S_\nu$
from even smaller values. To prove this, the
following \Cref{thm:success-upper-parts}  is crucial. It deals
with the self-reinforcing effect of
unwanted steps; more precisely, if probability mass is moved into $S_{\nu}\setminus S_{\nu+1}$, then
this increases the probability of biased
steps sampling a value in that interval
and increasing this mass even further.
However, Chernoff-type arguments can
be used to limit this effect. The
theorem is formulated in a general way
dealing with time-dependent success
probabilities and may be of independent
interest.

\begin{restatable}{theorem}{successUpperParts}
    \label{thm:success-upper-parts}
    Let $p\in [0,1]$ and $\eta,\rho \ge  0$.
    Let $X_t$, $t\ge 1$, be a sequence of independent Bernoulli trials. Let $Z_t\coloneqq \sum_{i=1}^t X_i$. Assume that
    \begin{equation}
        \label{eq:success-upper-parts:condition}
        p_t\coloneqq \prob{X_t=1\given Z_t} \ge \frac{\rho p+ \eta Z_{t}}{\rho + \eta t}.
    \end{equation}
    Let $0<\delta<1$ and assume $\tfrac{\eta t}{\rho+\eta t}\le b < 1$.  Then for  $\mu\coloneqq tp$ it holds that
    \[
        \prob{Z_t \le (1-\delta) \mu}  \le
        t \exp( -(1-b\delta)(1-b)^2\delta^2 \mu / 2).
    \]
\end{restatable}

We now present the crucial lemma dealing
with the relative frequency growth.

\begin{restatable}{lemma}{adequateBiasedGrowth}
    \label{lem:adequate-biased-growth}
    Recall the notation from \cref{eq:probabilityMassNotation,eq:intervalHierarchy}.
    Let $r \in \N_{\geq 3}$ with $r = \bigO{n^{1 / 6-\epsilon}}$
    for a constant $\epsilon \in \R_{>0}$, and let $\cstar, \cstop \in \R_{> 0}$ be sufficiently large constants.
    Consider the \rcga optimizing \globalOneMax, with a well-behaved $K \geq \cstar r \sqrt{n} \ln^2(n)\ln^2(r)$.
    Let $i \in [n]$, let $t' \in \N$ be a random iteration, and assume that there is a $\psi \in [0 .. \kappa^* - 1]$ such that $\mu_i^{(t')}(S_\psi)=1$ and $\mu_i^{(t')}(K_\psi) > 0$. Furthermore,  assume for all $\nu \in [\psi + 1 .. \kappa^* - 1]$ that $\mu^{(t')}_i(S_{\nu}) \ge (1-1/\kappa^*)^{3\kappa^*-3}/r$.
    Then, abbreviating $\tstop \leq \cstop K \sqrt{n} \ln n$, for each constant $\cphase \in \R_{> 0}$, we have for all $\nu \in [\psi + 1 .. \kappa^* - 1]$ and all $t\in [0..\tstop]$ that
    \begin{align*}
         & \prob*{
            \frac{\mu_{i}^{(t'+t)}(S_{\nu + 1})}
            {\mu_{i}^{(t'+t)}(S_{\nu})}
            <
            \left(1 - \frac{1}{\kappa^*}\right)^{3}
            \cdot   \frac{\mu_{i}^{(t')}(S_{\nu + 1})}
            {\mu_{i}^{(t')}(S_{\nu})}
            \given
            t'
        }                                                                                           \\
         & \quad \leq \tstop\exp(-\bigOmega{p n^{\epsilon/2}}) + 2\tstop n^{-\cstar /(3857\cstop)},
    \end{align*}
    where $p=\frac{(1-1/\kappa^*)\mu_{i}^{(t')}(S_{\nu + 1})}
        {(1+1/\kappa^*)\mu_{i}^{(t')}(S_{\nu})}$.
\end{restatable}

\subsection{Combining Everything}



The proof of \Cref{thm:rcga-on-gom} proceeds as explained in \Cref{sec:building-blocks}.
For space reasons, we only provide a proof sketch, recalling that we heavily rely on the notation from \cref{eq:intervalHierarchy,eq:probabilityMassNotation}.

As we explained before, the proof operates by considering~$\kappa^*$ phases, where a phase is defined per position $i \in [n]$ and per phase index $\kappa \in [0 .. \kappa^* - 1]$.
A phase starts in iteration $t' \in \N$ once~$K_\kappa$ is the lowest-index interval with non-zero probability, \ie, $\mu_i^{(t')}(S_\kappa) = 1$ and~$\kappa$ is the smallest index such that $\mu_i^{(t')}(K_\kappa) > 0$.
The phase ends once~$K_\kappa$ has a probability mass of~$0$.

For each phase but constantly many phases at the end, we apply \Cref{lem:driftBound}, assuming that the conditions for drift being present are satisfied for order $K \sqrt{n} \ln n$ iterations.
By an application of \Cref{thm:multiplicative-drift-concentration-with-failure}, we then obtain that the phase ends with high probability.
Via a union bound over all~$n$ values for~$i$, the same is true for all positions.
For the remaining phases, where we only consider constantly many different values, we apply \Cref{thm:rcga-runtime-adak-witt}, which is better than our bound in these cases.

Last, we show that the assumptions of \Cref{lem:driftBound} are met, which we do by applying \Cref{lem:adequate-biased-growth} inductively.
For the base case, the ratio of probability mass of neighboring $S$-intervals is just the ratio of the interval sizes, since all frequencies start at~$\frac{1}{r}$.
For the other cases, we apply \Cref{lem:contribution-random-walk-steps-kappastar}, showing that the probability mass in each $S$-interval we consider per phase is with high probability sufficiently large in order to apply \Cref{lem:adequate-biased-growth} inductively.
Thus, in each phase, we lose (per position) with high probability at most a fraction of $(1 - \frac{1}{\kappa^*})^3$ of the previous probability mass ratio.
Since we have at most~$\kappa^*-1$ phases before the final application of the lemma, we lose with high probability overall no more than fraction of $(1 - \frac{1}{\kappa^*})^{3 (\kappa^*-1)} \geq \eulerE[-3]$, which is a constant.
Hence, the condition of \Cref{lem:driftBound} is satisfied for all phases (and all positions).
Adding up all error probabilities via a union bound then yields our main result.

\section{Conclusion}
\label{sec:conclusion}

We provided an improved runtime guarantee of the \rcga on the \gonemax benchmark function.
In comparison to the previous result by \citet{AdakW25}, our result is better by a factor of order $r^2 / \bigl(\ln(n) \ln^2(r)\bigr)$.
Our main contribution is that we prove that the runtime does not scale cubicly in~$r$ but only linearly (up to polylogarithmic factors).
We achieve this by conducting a more careful analysis of how frequencies behave when subject to random updates and how likely it is for a frequency to receive a useful, biased update.

In comparison to the well-known result of the (binary) cGA on the (binary) \OneMax benchmark, we lose a logarithmic factor in~$n$ due to relying on a union bound for all positions via the multiplicative drift theorem.
In turn, this results in a larger lower bound for~$K$, which also affects the runtime.

Our scaling in~$r$ can likely be improved by combining phases more meticulously.
In the binary case, the potential function considers the distance of all highest-value frequencies to their maximum value, which results in variable drift.
In contrast, we still consider all positions independently.

With \Cref{thm:multiplicative-drift-concentration-with-failure,thm:success-upper-parts} and \Cref{lem:probabilityShiftingLowerBound}, we introduce in our analysis three general statements that are likely to be of independent interest.

Future work could focus on improving how to combine phases.
In particular, our crucial \Cref{lem:adequate-biased-growth} currently requires an upper bound on~$r$ that is very likely just an artifact of our analysis and can be improved.

\section*{Acknowledgments}
Martin~S. Krejca was supported by the French National Agency for Research (ANR) via the JCJC grant titled \emph{MultiEDA} (ANR-25-CE23-2418-01) and by the EuroTech Visiting Researcher program~2025. Carsten Witt was supported
by the  Independent Research Fund Denmark via grant 10.46540/2032-00101B.
This paper greatly benefitted from discussions at the \href{https://www.dagstuhl.de/25092}{Dagstuhl Seminar 25092}.

\printbibliography

\newpage
\appendix
\section{Appendix of ``Runtime Analysis of a Compact Genetic Algorithm on a Truly Multi-valued OneMax Function''}
\label{sec:appendix}

We provide all the proofs that are omitted in the main paper.
For the sake of convenience, we repeat the statements and order them the same way as in the main paper.

\subsection{Mathematical Tools}

\multiDriftWithFailure*

\begin{proof}
    We only prove the first part of the theorem because the second part follows immediately from the first when choosing $t = \ceil{(\ln(\expect{X_0 \cdot \indic{T > 0}} / \xmin) + \gamma) / \delta}$ and estimating $(1 - \delta)^t \leq \exp(-\delta t)$.

    Let $(Y_t)_{t \in \N} \coloneqq (X_t \cdot \indic{t < T})_{t \in \N}$ for convenience, and let $t \in [0 .. t']$.
    The inequality to show follows by noting that the event $\set{T > t}$ is by the definition of~$T$ equivalent to the event $\set{Y_t \geq \xmin}$.
    Making a case distinction with respect to the event $\set{\cap_{s \in [0 .. t]} A_t}$, we obtain
    \begin{align*}
        \prob{T > t}
         & = \prob{Y_t \geq \xmin}                                                                                                        \\
         & = \prob{Y_t \geq \xmin \land \cup_{s \in [0 .. t']} \overline{A_s}} + \prob{Y_t \geq \xmin \land \cap_{s \in [0 .. t']} A_s} .
    \end{align*}
    In the remainder of the proof, we bound the second probability from above.
    To this end, our proof follows the conventional one, as given, for example, by \citet[Theorem~$5$]{LenglerS18DriftRevisited}, noting that only events until time~$t$ are relevant.

    Note that since the event $\set{Y_t \geq \xmin \land \cap_{s \in [0 .. t']} A_s}$ is an intersection and since $\xmin > 0$ by assumption, the event is equivalent to $\set{Y_t \cdot \indic{\cap_{u \in [0 .. t']} A_u} \geq \xmin}$.
    For the convenience of notation, let $(Z_s)_{s \in \N} \coloneqq (Y_s \cdot \indic{\cap_{u \in [0 .. t']} A_u})_{s \in \N}$.
    By rearranging \cref{eq:multiplicative-drift-concentration-with-failure:drift-bound}, we obtain $\expect{Z_{t + 1} \given \filt{t}} \leq (1 - \delta) \expect{Z_t \given \filt{t}}$, and thus by the tower rule $\expect{Z_{t + 1}} \leq (1 - \delta) \expect{Z_t}$.
    Iterating this argument, we obtain $\expect{Z_t} \leq (1 - \delta)^t \expect{Z_0}$.
    Noting that~$Z$ is a nonnegative random process, Markov's inequality yields $\prob{Z_t \geq \xmin} \leq \frac{\expect{Z_t}}{\xmin}$.
    Applying the bound on~$\expect{Z_t}$ shows the first inequality.
    The second inequality follows by regarding a superset of the events of either term, thus concluding the proof.
\end{proof}

\subsection{Probability of Biased Steps}

\probabilityShiftingLowerBound*

\begin{proof}
    For the sake of brevity, let~$Q$ denote the left-hand side of \cref{eq:probabilityShiftingLowerBound}.

    We partition the domain~$[m \rho^+]$ of~$q$ into~$m$ subsets, defining for all $k \in [m]$ that $D_k = [(k - 1) \rho^+ + 1 .. k \rho^+]$, noting that $\bigcup_{k \in [m]} D_k = [m \rho^+]$.
    We then show the lower bound of~$\frac{M^2}{2 m}$ for another, easier sum that is contained in~$Q$ and only considers those terms where the two factors of~$q$ have common indices in one of the sets of our partition~$D$ above.

    To this end, we first double each product of~$q$ in~$Q$ except those with identical indices, and we observe that due to the commutativity of multiplication, we have
    \begin{align*}
         & 2 \sum\nolimits_{d \in [0 .. \rho]} \sum\nolimits_{i \in [m \rho^+ - d]} q(i) q(i + d) - \sum\nolimits_{i \in [m \rho^+]} q(i)^2 \\
         & = \sum\nolimits_{i \in [m \rho^+]} \sum\nolimits_{j \in [m \rho^+]} q(i) q(j) \cdot \indic{\abs{i - j} \leq \rho} .
    \end{align*}
    Solving for~$Q$ implies that
    \begin{equation}
        \label{eq:probabilityShiftingLowerBound:intermediateResult}
        Q
        \geq \frac{1}{2} \sum\nolimits_{i \in [m \rho^+]} \sum\nolimits_{j \in [m \rho^+]} q(i) q(j) \cdot \indic{\abs{i - j} \leq \rho} .
    \end{equation}
    The remaining part of the proof deals with bounding the outer sum on the right-hand side from below.

    As mentioned before, we derive another sum that only considers products of~$q$ whose indices share a set in the partition~$D$.
    Consequently, as indices in the same set of~$D$ are at most~$\rho$ apart and since each product of values of~$q$ is non-negative, due to the non-negativity of~$q$, we obtain
    \begin{align*}
         & \sum\nolimits_{i \in [m \rho^+]} \sum\nolimits_{j \in [m \rho^+]} q(i) q(j) \cdot \indic{\abs{i - j} \leq \rho} \\
         & \quad \geq \sum\nolimits_{k \in [m]} \sum\nolimits_{i \in D_k} \sum\nolimits_{j \in D_k} q(i) q(j) .
    \end{align*}

    In order to simplify the right-hand side above further, for all $k \in [m]$, let $m_k = \sum_{i \in D_k} q(i)$.
    Using this definition and rearranging the sums, since the indices are independent of each other, we obtain
    \begin{align*}
        \sum\nolimits_{k \in [m]} \sum\nolimits_{i \in D_k} \sum\nolimits_{j \in D_k} q(i) q(j)
         & = \sum\nolimits_{k \in [m]} \sum\nolimits_{i \in D_k} q(i) \sum\nolimits_{j \in D_k} q(j) \\
         & = \sum\nolimits_{k \in [m]} m_k^2 .
    \end{align*}

    By the Cauchy--Schwarz inequality, we furthermore see that
    \begin{equation*}
        \sum\nolimits_{k \in [m]} m_k^2
        \geq \frac{1}{m} \bigl(\sum\nolimits_{k \in [m]} m_k\bigr)^2 .
    \end{equation*}

    The proof then follows by recalling that \cref{eq:probabilityShiftingLowerBound:intermediateResult} requires a factor of~$\frac{1}{2}$ as well as that the statement assumes that $\sum_{i \in [m \rho^+]} q(i) = M$, which, since~$D$ is a partition, is equivalent to $\sum_{k \in [m]} m_k = M$.
\end{proof}

\biasedStepLowerBound*

\begin{proof}
    We consider the interval $I \coloneqq \bigl[\floor{\expect{\sumOfSampleValues{t'}{1}{i} \given \freq{t'}{}{}} - 2 \sigma} + 1 .. \ceil{\expect{\sumOfSampleValues{t'}{1}{i} \given \freq{t'}{}{}} + 2 \sigma} - 1\bigr]$ and then apply \Cref{lem:probabilityShiftingLowerBound}.
    First, note that by the assumption on the bounds of~$I$, we have $I \subseteq [0 .. (n - 1) (r - 1)]$.

    For each $d \in [0 .. \Delta]$, let $I_{-d} = \bigl[\floor{\expect{\sumOfSampleValues{t'}{1}{i} \given \freq{t'}{}{}} - 2 \sigma} + 1 .. \ceil{\expect{\sumOfSampleValues{t'}{1}{i} \given \freq{t'}{}{}} + 2 \sigma} - 1 - d\bigr] \subseteq I$.
    By the definition of the quantities from \cref{eq:sumOfSamples}, we see that
    \begin{align*}
         & \prob[\big]{\differenceOfSampleValues{t'}{i} \in [0 .. \Delta]}                                                                                                                                             \\
         & \quad \geq \sum\nolimits_{d \in [0 .. \Delta]} \sum\nolimits_{v \in I_{-d}} \prob{\sumOfSampleValues{t'}{1}{i} = v \given \freq{t'}{}{}} \prob{\sumOfSampleValues{t'}{2}{i} = v + d \given \freq{t'}{}{}} .
    \end{align*}
    Note that for each $d \in [0 .. \Delta]$, we can shift the index~$v$ of the second sum such that it ranges over the interval at least $[4 \sigma - 1 - d]$, which is non-empty by the assumption on~$\sigma$.
    Hence, applying \Cref{lem:probabilityShiftingLowerBound} with $\rho = \Delta$ and $m = \floor{(4 \sigma - 1) / \max \set{1, \Delta}} \geq 1$, and noting that~$\sumOfSampleValues{t'}{1}{i}$ and~$\sumOfSampleValues{t'}{2}{i}$ follow the same distribution, we obtain
    \begin{align*}
        \prob[\big]{\differenceOfSampleValues{t'}{i} \in [\Delta]}
        \geq \frac{\prob{\sumOfSampleValues{t'}{1}{i} \in I \given \freq{t'}{}{}}^2}{2 \floor{(4 \sigma - 1) / \max \set{1, \Delta}}}
        \geq \frac{\prob{\sumOfSampleValues{t'}{1}{i} \in I \given \freq{t'}{}{}}^2 \cdot \max \set{1, \Delta}}{2 (4 \sigma - 1)}
        .
    \end{align*}
    Performing the following estimations and applying Cheybshev's inequality, accounting for rounding, yields
    \begin{align*}
        \prob[\big]{\sumOfSampleValues{t'}{1}{i} \in I \given \freq{t'}{}{}}
         & \geq 1 - \prob[\Big]{\abs[\big]{\sumOfSampleValues{t'}{1}{i} - \expect{\sumOfSampleValues{t'}{1}{i} \given \freq{t'}{}{}}} \geq 2 \sigma \given \freq{t'}{}{}} \\
         & \geq 1 - \frac{1}{4}
        = \frac{3}{4} ,
    \end{align*}
    which concludes the proof.
\end{proof}

\varianceUpperBound*

\begin{proof}
    Note that~$\sumOfSampleValues{t'}{1}{i}$ and~$\sumOfSampleValues{t'}{2}{i}$ follow the same distribution, which is why it is sufficient to focus on one of them.
    To this end, let~$S$ denote a random variable that follows the same distribution as both~$\sumOfSampleValues{t'}{1}{i}$ and~$\sumOfSampleValues{t'}{2}{i}$.

    Since the components of~$\sample{t'}{1}{}$ are independent and thus uncorrelated, it follows that $\var{S \given \freq{t'}{}{}} = \sum_{k \in [n] \smallsetminus \set{i}} \var{\sample{t'}{1}{k} \given \freq{t'}{}{}}$.

    Let $k \in [n] \smallsetminus \set{i}$.
    We use $\var{\sample{t'}{1}{k} \given \freq{t'}{}{}} = \expect{(\sample{t'}{1}{k} - \expect{\sample{t'}{1}{k} \given \freq{t'}{}{}})^2 \given \freq{t'}{}{}}$.
    To this end, note that by the definition of~$\jStar$, we have $\expect{\sample{t'}{1}{k} \given \freq{t'}{}{}} \geq \jStar$.
    Thus, the distance $\abs{\sample{t'}{1}{k} - \expect{\sample{t'}{1}{k} \given \freq{t'}{}{}}}$ of each outcome of~$\sample{t'}{1}{k}$ for values in $[\jStar .. r - 1]$ is at most $r - 1 - \jStar$.
    Consequently, we estimate $\var{\sample{t'}{1}{k} \given \freq{t'}{}{}} \leq (r - 1 -  \jStar)^2 $.
    We conclude the proof by noting that we consider in total $n - 1$ positions, that is, values for~$k$.
\end{proof}

\subsubsection{Large Biased Steps and Drift Bound}

\driftBound*

\begin{proof}
    The overall aim is to bound the probability $\prob{B}$ of a biased step at position~$i$, where we drop the index~$i$ for convenience. We  let
    $G$ be the event that a large step occurs at position~$i$, \ie, one sample has a value of at least~$\ell_{\kappa+2}$
    and the other one a value less than $\ell_{\kappa+1}$ in position~$i$.
    By the assumptions
    of the lemma, we have $\prob{G} \ge \sTstar{t}(1-\sDstarP{t})  \ge \cdrift \sDstarP{t}(1-\sDstarP{t})$. On~$G$, a biased
    step occurs if $\card{\differenceOfSampleValues{t'}{i}} < 5(r-1-\ell_\kappa)/9 - 1 \le \ell_{\kappa+2} - \ell_{\kappa}$, where the $-1$ is an estimate on the rounding of the~$\ell_\kappa$. We set $\Delta$ smaller than $5(r-1-\ell_\kappa)/9 - 1$ by choosing $\Delta = (r-1-\ell_\kappa)/2$ (using $\ell_k \le r-10$) in Lemma~\ref{lem:biasedStepLowerBound} and also  satisfy $(\Delta + 2) / 4 \le \sigma$
    by choosing $\sigma$ large enough. The lemma now bounds
    the probability of a biased
    step at position~$i$ as
    \begin{align*}
        \prob{B} & \ge \prob{G} \prob[\big]{\differenceOfSampleValues{t'}{i} \in [0 .. \Delta]} \ge \cdrift\cdot \sDstarP{t}(1-\sDstarP{t}) \frac{9 \max \set{1, \Delta}}{32 (4 \sigma - 1)}
        \\ & \ge
        \frac{9\cdrift \cdot \sDstarP{t}(1-\sDstarP{t}) \max \set{1, (r-1-\ell_\kappa)/2}}{32 (4 \sigma - 1)} .
    \end{align*}
    Bounding $\sigma$ via Lemma~\ref{lem:varianceUpperBound} with $\jStar=\ell_k$, noting
    that there is no probability mass in $[0..\ell_\kappa-1]$,
    we have
    \[
        \prob{B} \ge    \frac{9\cdrift\cdot \sDstarP{t}(1-\sDstarP{t}) \max \set{1, (r-1-\ell_\kappa)/2}}{32 \Biggl(4
            \sqrt{(n - 1) (r - 1 - \ell_\kappa)^2 } - 1\Biggr)} \eqqcolon D_i^{(t')}(\kappa)
    \]
    The considered biased step will increase a frequency for one of the values in $S_{\kappa+2}$
    of the values in $[0..r-1]\setminus S_{\kappa+1} = [0,\dots,\ell_{\kappa+1}-1]$. Hence,
    \[
        \expect{\sTstar{t'+1} - \sTstar{t'}\given \freq{t'}{}{}}
        \geq \frac{D_i^{(t')}(\kappa) }{K} .
    \]

    The same bound holds for $\expect{\sDstarP{t'+1} - \sDstarP{t'}\given \freq{t'}{}{}}$ since
    the considered steps, which sample a value of at
    least~$\ell_{\kappa+2}$ and a value smaller than $\ell_{\kappa+1}$, increase both $\sDstarP{t'}$ and $\sTstar{t'}$ at the same time but decrease $(1-\sDstarP{t'})$.
\end{proof}

\subsection{Bounds for the Change of Frequency Sums in Phases}

\subsubsection{Bounds on Frequency Decrease due to Genetic Drift}

\rcgaNeutralConcentration*

\begin{proof}
    We aim to apply \Cref{thm:martingaleConcentration} to $(P'_t)_{t \in \N} \coloneqq (P_{U_t})_{t \in \N}$.
    To this end, we only consider a restriction of~$P'$ until it deviates from~$P'_0$ by at least~$\alpha P'_0$, noting that the statement is identical for the original process and for its restriction.

    Let $V = \inf \set{t \in \N \given P'_t \notin \bigl((1 - \alpha) P'_0, (1 + \alpha) P'_0\bigr)}$, and for all $t \in \N$, let $X_t = P'_t \cdot \indic{t < V} + P'_S \cdot \indic{t \geq V}$ be the restriction of~$P'$ that we are interested in, that is, the stopped process of~$P'$.
    Note that since~$P'$ is a martingale by assumption, we see that~$X$ is a martingale with respect to the natural filtration of $(\freq{t}{}{})_{t \in \N}$.

    We aim to apply \Cref{thm:martingaleConcentration} to~$X$, choosing $\varepsilon = \alpha P'_0$.
    Since the update of the \rcga changes a frequency by at most~$\frac{1}{K}$, it follows for \Cref{thm:martingaleConcentration} that $b = \frac{1}{K}$.
    It remains to estimate for all $t \in \N$ the value $\hat{v}_t$.

    Let $t \in \N$.
    We derive $\var{X_{t + 1} - X_t \given \freq{t}{i}{}}$ via a case distinction with respect to how~$t$ compares to~$V$.

    \textbf{Case $\boldsymbol{t < V}$.}
    Since the frequencies in~$X$ are not cut off by the update, we substitute $\freq{t + 1}{i}{j} = \freq{t}{i}{j} + \frac{1}{K} (\indic{x_i = j} - \indic{y_i = j})$.
    Using that the two indicator functions per update are independent and thus uncorrelated, we obtain
    \begin{align*}
        \var{X_{t + 1} - X_t \given \freq{t}{i}{}}
         & \leq \var*{\frac{1}{K}(\indic{\selectedSample{t}{1}{i} = j} - \indic{\selectedSample{t}{2}{i} = j})}                                                         \\
         & = \frac{1}{K^2}\bigl(\var{\indic{\selectedSample{t}{1}{i} = j} \given \freq{t}{i}{}} + \var{\indic{\selectedSample{t}{2}{i} = j} \given \freq{t}{i}{}}\bigr) \\
         & = \frac{2}{K^2} \freq{t}{i}{j} (1 - \freq{t}{i}{j})                                                                                                          \\
         & \leq \frac{2}{K^2} \min \set{\freq{t}{i}{j}, 1 - \freq{t}{i}{j}}.
    \end{align*}
    Since we consider a~$t$ such that $\freq{t}{i}{j} \in \bigl[(1 - \alpha) P'_0, (1 + \alpha) P'_0\bigr]$, we have $\var{X_{t + 1} - X_t \given \freq{t}{i}{}} \leq \frac{2}{K^2} \beta$.

    \textbf{Case $\boldsymbol{t \geq V}$.}
    Since~$X$ no longer changes once $t \geq V$, it follows that $\var{X_{t + 1} - X_t \given \freq{t}{i}{}} = 0$.

    \textbf{Concluding the first case.}
    Combining the two cases, we see for all $t \in \N$ that $\hat{v}_t \leq t \frac{2}{K^2} \beta$.
    Applying \Cref{thm:martingaleConcentration} yields for all $t \in \N$ that
    \begin{equation*}
        \prob[\big]{\exists s \in [0 .. t]\colon \abs*{P'_s - P'_0} \geq \alpha P'_0 \given P'_0}
        \leq 2 \exp\left(-\frac{(\alpha P'_0)^2}{2 \cdot 2 t \beta / K^2 + 2 \alpha P'_0 / (3 K)}\right) .
    \end{equation*}
    Simplifying the exponent on the right-hand side yields
    \begin{equation*}
        -\frac{(\alpha P'_0)^2}{2 \cdot 2 t \beta / K^2 + 2 \alpha P'_0 / (3 K)}
        \leq -\frac{3 (\alpha P'_0 K)^2}{4 \cdot \max \{6 t \beta, \alpha P'_0 K\}} ,
    \end{equation*}
    which concludes the proof, as the statement is the same for~$P'$ and for~$X$.
\end{proof}

\contributionRandomWalkStepsKappastar*

\begin{proof}
    We aim at applying \Cref{thm:rcgaNeutralConcentration} to $(P_t)_{t \in \N} \coloneqq \bigl(\mu_{i \mid \randomWalkStoppingTime{i}{t'}[]}^{(t' + t)}(S_{\nu})\bigr)_{t \in \N}$ with $\alpha = \frac{1}{\kappa^*}$, that is, choosing in the theorem $J = S_{\nu}$.
    Note that by the definition of~$\randomWalkStoppingTime{i}{t'}[]$, we only consider random-walk steps in~$P$.
    Hence,~$P$ is a martingale.
    That is, \Cref{thm:rcgaNeutralConcentration} is directly applicable.
    This leaves us only with discussing the exponent in the probability bound in the following.

    For~$\beta$, we see that our assumption
    $\kappa^* \ge 2$ implies $\alpha \le 1/2$, so we  have  $ \frac{1}{2} P_0 \le \beta \le \frac{3}{2}P_0$
    for any value of~$P_0$.

    From the definition of~$\tstop$ and the bound on~$\beta$, we have $6 \tstop \beta \ge 3 \tstop P_0 \geq 3\cstop K \sqrt{n} (\ln n)  P_0 \ge \alpha K  P_0$ for
    sufficiently large~$n$.
    Thus, we get for the exponent of \Cref{thm:rcgaNeutralConcentration} that
    \begin{align*}
        -\frac{3 (\alpha P_0 K)^2}{4 \cdot \max \{6 \tstop \beta, \alpha P_0 K\}}
         & = -\frac{3 \alpha^2 P_0^2 K^2}{24 \tstop \beta}
        \le  -\frac{3 \alpha^2 P_0^2 K^2}{24\cstop K\sqrt{n}(\ln n)(3/2)P_0} \\
         & = -\frac{\alpha^2 P_0 K}{12\cstop \sqrt{n}\ln n}.
    \end{align*}
    We substitute~$K$ and recall that $P_0 \geq \frac{\eulerE[-3]}{r}$ and $\alpha = \frac{1}{\kappa^*} = \tfrac{1}{\lceil \log_{3/2}(r-1)\rceil} \ge \tfrac{1}{1.6\log_{3/2}(r-1)}\ge \frac{1}{4\ln r}$, where the first inequality holds for $r\ge 3$. Hence, we obtain an upper bound for the exponent of $-\frac{\cstar  \ln n}{192\eulerE[3]\cstop} \le -\frac{\cstar  \ln n}{3857\cstop}$, which concludes the proof.
\end{proof}

\subsubsection{Bounds on Increases due to Biased Steps}

\successUpperParts*

\begin{proof}
    We derive a lower bound $p_i\ge (1-b\delta )p$  that holds for all $i\le t$ with high probability and then apply the usual Chernoff bounds on independent trials with this pessimistic bound on their success probability.

    If $p_i\ge (1-b \delta)p$ for all $i\le t$, then by standard Chernoff bounds
    with success probability $(1-b \delta)p$ and expectation $t(1-b \delta)p$, we have
    \begin{align*}
        \prob{Z_t\le (1-\delta) tp} & \le
        \prob{Z_t \le (1-b \delta)(1-(1-b)\delta)tp}                               \\
                                    & \le \exp(-(1-b\delta)(1-b)^2\delta^2 tp / 2)
    \end{align*}

    We now study the event that $p_i \ge (1-b\delta) p $ for all $i\le t$.
    We note that our bound $\frac{\rho p+\eta S_i}{\rho+\eta i}$ on~$p_i$ becomes smallest for $i=t$, pessimistically assuming $Z_i \le ip$ (otherwise we already have $p_i \ge p$). If
    $Z_t\ge (1-\delta)tp$, then
    \[
        p_i \ge \frac{\rho p+\eta (1-\delta) tp}{\rho +\eta t}
        = \left(1-\frac{\delta \eta t}{\rho+\eta t}\right) p \ge
        (1-b\delta) p ,
    \]
    where the last inequality follows from our assumption $\tfrac{\eta t}{\rho+\eta t}\le b$.
    Hence, by a union bound over~$t$ values, the probability that there is an $i\le t$ where
    $p_i< (1-b\delta)p$ is at
    most $t\exp(-(1-b\delta)(1-b)^2\delta^2 tp / 2)$. Otherwise, the Chernoff bound for $Z_t$ applies, whose failure probability has
    already been subsumed in the union bound. Hence, the total failure probability is
    still bounded as before.
\end{proof}

\adequateBiasedGrowth*

\begin{proof}
    The proof essentially analyzes how biased steps
    add probability mass to either $K_\nu$ or $S_{\nu+1}$.
    In particular, we consider only biased steps that increase a frequency in~$S_\nu$.
    Typically, only a tiny fraction of the steps in the
    considered time interval is biased, so
    it is crucial to bound the total number
    of biased steps.

    First, we filter the biased steps given by \cref{eq:biased-stopping-times} from iteration~$t'$ onward, namely, $(\biasedStoppingTime{i}{t'}[t])_{t \in \N}$, to consider only those where a frequency in~$S_\nu$ is increased.
    To this end, let $(U^i_{t', t})_{t \in \N}$ be defined as $U^i_{t', 0} = \biasedStoppingTime{i}{t'}[0] = t'$ and, for all $t \in \N_{\geq 1}$, that $U^i_{t', t} = \inf\set{\biasedStoppingTime{i}{t'}[s] > U^i_{t', t - 1} \given s \in \N \land \Delta^{(\biasedStoppingTime{i}{t'}[s])}_i(S_\nu) > 0}$.
    Since we aim to restrict the amount of such biased steps, let 
    $B_t =
        \max\set{t'' \in \N \given U^i_{t', t''} \leq t' + t}$ be
    the number of these steps
    in the time interval $[t',t'+t]$
    and $B=B_{\tstop}$ be the number
    in the whole time interval considered.
    We claim that $B\le \left(1+\tfrac{1}{\kappa^*}\right)K$ holds with
    high probability, and we postpone the proof for later.

    We aim at applying \Cref{thm:success-upper-parts} to the biased steps that increase a frequency in~$S_{\nu + 1}$, considering all biased steps that increase a frequency in~$S_\nu$, that is, using \cref{eq:contribution-delta}, to
    \begin{equation*}
        (\indic{\Delta^{(U^i_{t', t})}_i(S_{\nu + 1}) > 0})_{t \in [B]}
        \eqqcolon (X_t)_{t \in [B]} .
    \end{equation*}
    To this end, we choose $\eta = \frac{1}{K}$, $\rho = (1+\frac{1}{\kappa^*}) \mu_{i}^{(t')}(S_\nu)$, and recall $p = \frac{1-1/\kappa^*}{1+1/\kappa^*} \frac{\mu_{i }^{(t')}(S_{\nu + 1})}{\mu_{i}^{(t')}(S_\nu)}$. We note that $\rho p = (1-\frac{1}{\kappa^*} ) \mu_{i}^{(t')}(S_{\nu + 1})$, and we recall that $\biasedStoppingTime{i}{t'}[0] = t'$ by definition.

    In order to apply \Cref{thm:success-upper-parts}, we show for all $t \in [B]$ that \cref{eq:success-upper-parts:condition} holds.
    Using the same notation of~$Z$ from \Cref{thm:success-upper-parts}, by the definition of~$X$, which only considers points in time where we increase a frequency in~$S_\nu$, we obtain for all $t \in [0 .. B]$ that
    \begin{equation}
        \label{lem:adequate-biased-growth:actual-update-probability}
        \prob{X_{t + 1} = 1 \given Z_{t + 1}, U^i_{t', t}}
        = \frac{\mu_{i}^{(t' + U^i_{t', t})}(S_{\nu + 1})}{\mu_{i}^{(t' + U^i_{t', t})}(S_\nu)} .
    \end{equation}
    We bound \cref{lem:adequate-biased-growth:actual-update-probability} in the following from below by bounding the numerator and denominator separately.


    \textbf{Bounding the numerator from below.}
    The contribution to the probability mass in~$S_{\nu + 1}$ does not only consist of the increase via biased steps but also of the updates via (typically
    much more frequent) random-walk steps, which can decrease the probability mass.
    We account for the decrease via random-walk steps by \Cref{lem:contribution-random-walk-steps-kappastar}.
    Since we assume that $\mu^{(t')}_i(S_{\nu + 1}) \ge (1-1/\kappa^*)^{3\kappa^*-3}/r \geq \frac{\eulerE[-3]}{r}$ and that $\tstop \leq \cstop K \sqrt{n} \ln n$, \Cref{lem:contribution-random-walk-steps-kappastar} is applicable and yields that genetic drift contributes to the frequencies in~$S_{\nu + 1}$ at least a value of $-\frac{1}{\kappa^*} \mu_{i \mid \randomWalkStoppingTime{i}{t'}[]}^{(t')}(S_{\nu + 1})$ with probability at least $1 - 2 n^{-\cstar / (3857 \cstop)}$.

    Since the biased steps cannot decrease the frequencies in~$S_{\nu + 1}$, and since $U^i_{t', 0} = \biasedStoppingTime{i}{t'}[0]$, we get for all $t \in [0 .. \tstop]$ with probability at least $1 - 2 n^{-\cstar / (3857 \cstop)}$ that
    \begin{equation*}
        \mu_{i}^{(t' + U^i_{t', t})}(S_{\nu + 1})
        \geq \left(1-\frac{1}{\kappa^*}\right) \mu_{i \mid U^i_{t'}}^{(t')}(S_{\nu + 1}) + \frac{Z_{t + 1}}{K}
        = \rho p + \eta Z_{t + 1} .
    \end{equation*}

    \textbf{Bounding the numerator from above.}
    Similarly to the previous case, the probability mass in~$S_\nu$ is updated via random-walk steps and increased via biased steps.
    By applying \Cref{lem:contribution-random-walk-steps-kappastar}, we obtain that the increase via random-walk steps is at most $\frac{1}{\kappa^*}$ of the starting value, with probability at least $1 - 2 n^{-\cstar / (3857 \cstop)}$.

    Moreover, each biased update in~$S_\nu$ increases the probability mass by~$\frac{1}{K}$.
    Thus, for all $t \in [0 .. \tstop]$, we have
    \begin{equation*}
        \mu_{i}^{(t' + U^i_{t', t})}(S_\nu)
        \leq \left(1+\frac{1}{\kappa^*}\right) \mu_{i \mid U^i_{t'}}^{(t')}(S_\nu) + \frac{t + 1}{K}
        = \rho + \eta (t + 1) .
    \end{equation*}

    \textbf{Bounding the ratio
        of mass at time~$t$.}
    The two inequalities above show that \cref{eq:success-upper-parts:condition} is satisfied. Recalling that we are aiming at a statement holding for all time steps $t'+t\in[t'..t'+\tstop]$, we
    distinguish two cases according
    to the number~$B_t$ of biased steps increasing
    $S_{\nu}$ in that time interval.

    \textbf{Case 1: $B_t\ge  (1-1/\kappa^*)^{3(\kappa^*-1)} \frac{K}{\kappa^* r}$.} In this case,
    $B_t$ is so large that
    \Cref{thm:success-upper-parts}
    yields sufficient concentration for the probability mass moved
    into $S_{\nu+1}$.


    Using the notation from \Cref{thm:success-upper-parts}, we choose $\delta = \frac{1}{\kappa^*}$ as well as $b = 1-\frac{1}{2\eulerE[3]r}$, noting that, due to our assumptions
    $B_t \le B \leq (1+\frac{1}{\kappa^*})K$ and $\mu_{i}^{(t')}(S_\nu) \ge (1-1/\kappa^*)^{3\kappa^*-3}/r \geq \frac{\eulerE[-3]}{r}$, we obtain
    \begin{equation*}
        b
        = 1 - \frac{1}{2 \eulerE[3] r }
        \geq 1 - \frac{1}{1 + \eulerE[3] r}
        \ge 1 - \frac{1}{1 + \eta B_t/\rho}
        = \frac{\eta B_t}{\rho + \eta B_t} .
    \end{equation*}

    Hence, by applying \Cref{thm:success-upper-parts}, we have that $Z_{B_t + 1} \leq (1 - \delta) B_t p = (1-\frac{1}{\kappa^*}) B_tp$ only occurs with probability at most
    \begin{align*}
        B_t \exp\left(-\frac{1}{2} (1 - b \delta) (1 - b)^2 \delta^2 B_t p\right)
         & = B_t \exp\left(- \bigOmega*{\frac{\delta^2 B_t p}{r^2 }}\right) \\ & = B_t \exp\left(- \bigOmega*{\frac{\delta^2 K p}{r^3\ln r }}\right)
    \end{align*}
    by our assumption $B_t\ge (1-1/\kappa^*)^{3(\kappa^*-1)} \frac{K}{\kappa^* r}  = \Omega(K/(r\ln r))$. Further using that
    $r=O(n^{1/6-\epsilon})$, $K=\Omega(\sqrt{n})$, $B_t=\mathrm{poly}(n)$,
    and $\delta=\Omega(1/\!\ln n)$, the
    bound on the failure probability is at most $\exp(-\Omega(n^{\epsilon/2}p))$.

    \textbf{Concluding under the current assumptions.}
    Under the likely event $Z_{B_t + 1} \geq (1 - \delta) B_t p$, we
    now relate the final probability masses in $S_{\nu+1}$ and
    $S_\nu$ to each other to prove the statement of the theorem.
    The event states
    that at least $(1-1/\kappa^*)B_t p$ of the $B_t$ specific biased steps in the
    interval $[t'..t'+t]$ add
    probability mass to~$S_{\nu+1}$. We assume this to happen and   pessimistically let
    all other biased steps increase the mass of $S_{\nu}$. Hence,
    we have at time~$t$ that
    \[
        \mu_{i}^{(t'+t)}(S_{\nu + 1})
        \geq \left(1-\frac{1}{\kappa^*}\right) \mu_{i}^{(t')}(S_{\nu + 1}) + \frac{(1-1/\kappa^*)B_t p}{K}
    \]
    and
    \[
        \mu_{i}^{(t'+t)}(S_{\nu})
        \leq \left(1+\frac{1}{\kappa^*}\right) \mu_{i}^{(t')}(S_{\nu}) + \frac{B_t(1-(1-1/\kappa^*))}{K}
        \leq \left(1+\frac{1}{\kappa^*}\right) \mu_{i}^{(t')}(S_{\nu}) + \frac{B_t}{K}.
    \]
    Recalling that $p = \frac{(1-\frac{1}{\kappa^*}) \mu_{i}^{(t')}(S_{\nu + 1})}{(1+\frac{1}{\kappa^*}) \mu_{i}^{(t')}(S_{\nu})}$,
    combining the two inequalities, we obtain
    \[
        \frac{\mu_{i}^{(t'+t)}(S_{\nu + 1})}{ \mu_{i}^{(t'+t)}(S_{\nu})}
        \ge \frac{p (1+1/\kappa^*) \mu_{i}^{(t')}(S_{\nu}) +p(1-1/\kappa^*)B_t/K}{(1+1/\kappa^*) \mu_{i}^{(t')}(S_{\nu}) + B_t/K}
        \ge p(1-1/\kappa^*).
    \]
    Again substituting $p$ in the last bound and noting that
    $\frac{1-1/\kappa^*}{1+1/\kappa^*}\ge (1-\frac{1}{\kappa^*})^2$, we finally have
    \[
        \frac{\mu_{i}^{(t'+t)}(S_{\nu + 1})}{ \mu_{i}^{(t'+t)}(S_{\nu})}
        \ge
        \left(1-\frac{1}{\kappa^*}\right)^3  \cdot
        \frac{\mu_{i}^{(t')}(S_{\nu + 1})}{ \mu_{i}^{(t')}(S_{\nu})}
    \]
    with the claimed probability, recalling that we still need to bound the probability of~$B$ not being too large, which we do at the end of this proof.

    \textbf{Case 2: $B_t <  (1-1/\kappa^*)^{3(\kappa^*-1)}$.}
    Here the number of steps is so small that Theorem~\ref{thm:success-upper-parts} does not provide
    strong enough concentration. However, then
    the few biased steps cannot significantly change the ratio of masses
    in $S_\nu$ and $S_{\nu+1}$. We pessimistically assume that
    all the $S_\nu$-increasing biased steps in the time
    interval $[t'..t'+t]$ increase $S_\nu\setminus S_{\nu+1}$. Similarly to the above but
    without applying Theorem~\ref{thm:success-upper-parts}, we
    have
    \begin{align*}
        \frac{\mu_{i}^{(t'+t)}(S_{\nu + 1})}{ \mu_{i}^{(t'+t)}(S_{\nu})}
         & \ge \frac{p (1+1/\kappa^*) \mu_{i}^{(t')}(S_{\nu}) }{(1+1/\kappa^*) \mu_{i}^{(t')}(S_{\nu}) + B_t/K}
        \ge \frac{p \mu_{i}^{(t')}(S_{\nu}) }{\mu_{i}^{(t')}(S_{\nu}) + B_t/K}                                        \\
         & \ge \frac{p \mu_{i}^{(t')}(S_{\nu}) }{\mu_{i}^{(t')}(S_{\nu}) + (1-1/\kappa^*)^{3\kappa^*-3}/(\kappa^* r)} \\
         & \ge \frac{p (1-1/\kappa^*)^{3\kappa^*-3}/r}{
            (1-1/\kappa^*)^{3\kappa^*-3}/r + (1-1/\kappa^*)^{3\kappa^*-3}/(\kappa^* r)
        } \ge
        \left(1-\frac{1}{\kappa^*}\right) p .
    \end{align*}
    using the prerequisite $\mu_i^{(t')}(S_\nu)\ge (1-1/\kappa^*)^{3\kappa^*-3}/r$ from the lemma.

    \textbf{Bounding the total number of biased steps.} We are still left with the claim that $B\le \left(1+\frac{1}{\kappa^*}\right)K$ with high probability. This is actually a consequence
    of the fact that we have bounded the number of random-walk
    steps with high probability. First, note that $\mu_i^{(t')}(K_{\psi})\le 1$, $\mu_i^{(t'+\tstop)}(K_{\psi})= 0$ by
    definition of~$\tstop$, and that
    each biased step decreases the probability
    mass in $K_{\psi}$ by $1/K$. Hence,
    the number of biased steps in the time interval $[t'..t'+\tstop]$ is bounded by $\mu_i^{(t'+\tstop)}(K_{\psi}) - \mu_i^{(t')}(K_{\psi}) \le K$, plus
    $K$ times the accumulated
    increase of probability mass in $K_{\psi}$ due to random-walk steps.
    We recall the application
    of \Cref{lem:contribution-random-walk-steps-kappastar} above
    and obtain that genetic drift increases the probability mass in $K_\psi$ by at most $\frac{1}{\kappa^*} \mu_i^{t'} (K_{\psi}) \le \frac{1}{\kappa^*}$ in the time interval
    $[t'..t'+\tstop]$, with the high probability $1-2n^{-c^*/(3857\cstop)}$ stated above.  Hence,
    there are at most $K(1+1/\kappa^*)$ biased steps with
    the stated probability.
    Adding this probability to the failure probability of the statement and taking
    a union bound over the at most $\tstop$ steps considered concludes the proof.
\end{proof}

\subsection{Combining Everything}

\rcgaOnGom*

\begin{proof}
    We make a case distinction with respect to~$r$.
    If~$r$ is constant, we apply the known result by \citet[Theorem~$4.6$]{AdakW25}, which is better than our runtime bound for such values of~$r$.
    Hence, in the following, we assume that $r = \smallOmega{1}$.

    Recalling \cref{eq:intervalHierarchy} and as mentioned before, we split the analysis into $\kappa^* = \ceil{\log_{3/2}(r - 1)}$ consecutive \emph{phases} for each position $i \in [n]$.
    At the beginning of phase $\kappa \in [0 .. \kappa^* - 1]$, all probability mass at position~$i$ is in frequencies for values in~$S_\kappa$ and none of the other frequencies, and the probability mass in~$K_\kappa$ is positive.
    Phase~$\kappa$ stops once all probability mass from~$K_\kappa$ is removed.
    We note that a phase $\kappa' \in [\kappa + 1 .. \kappa^* - 1]$ is skipped if for all $\psi \in [\kappa + 1 .. \kappa']$, the frequency mass in~$K_{\psi}$ is (also)~$0$ at the end of phase~$\kappa$.

    We assume first that the assumptions of \Cref{lem:driftBound} are met during an entire phase.
    Via the multiplicative drift theorem (\Cref{thm:multiplicative-drift-concentration-with-failure}), we show that each phase (for each position) lasts with high probability at most order $K \sqrt{n} \ln n$ iterations.
    Since each position is optimized in parallel, a union bound over all positions yields that a single phase for all positions also only takes at most order $K \sqrt{n} \ln n$ iterations.
    Since we have~$\kappa^* = \bigO{\log r}$, the total runtime is with high probability at most $\bigO{K \sqrt{n} \log(n) \log(r)}$.
    Afterward, we show via \Cref{lem:adequate-biased-growth} that the assumptions of \Cref{lem:driftBound} hold for each phase and each position, thus concluding the proof.

    We note that \Cref{lem:driftBound} only works for phases $\kappa \in [0 .. \kappa^* - 1]$ where $\ell_\kappa \leq r - 10$, which excludes constantly many phases toward the end.
    For these phases, we apply the result by Adak and Witt (\Cref{thm:rcga-runtime-adak-witt})  directly, noting that we only consider a constant amount of frequencies with positive mass, all of which have at the beginning of each remaining phase at least a constant value.
    Although this does not imply that all these frequencies share their mass uniformly at the beginning of a phase, the result from \Cref{thm:rcga-runtime-adak-witt} is still applicable, since scaling the initial value of all involved frequencies by a constant factor does not affect the results asymptotically by more than a constant factor.
    This implies that we have the same bound for all phases.

    \textbf{Bounding the runtime for each phase for each position.}
    Consider position $i \in [n]$ and phase $\kappa \in [0 .. \kappa^* - 1]$.
    Let $\cstar \in \R_{> 0}$ be a sufficiently large constant, and let $t' \in \N$ be the random iteration in which phase~$\kappa$ starts for position~$i$.
    Assume that the assumptions of \Cref{lem:driftBound} are met for the first $\cstar K \sqrt{n} \ln n$ iterations of the phase or until the phase ends.
    To this end, let~$\mathcal{A}_{t'}$ denote the event that the assumptions of \Cref{lem:driftBound} hold for the first $\cstar K \sqrt{n} \ln n$ iterations of the phase.
    We note that we show at the end of this proof that $\prob{\mathcal{A}_{t'}} = \bigO{n^{-3}}$.
    Moreover, assume that the probability mass in~$S_{\kappa + 1}$ is at least some constant throughout this time.

    We consider the random process $(X_t)_{t \in \N} \coloneqq (1 - \sDstarP{t' + t})_{t \in \N}$.
    Let $c_1, c_2 \in \R_{> 3}$ be sufficiently large constants, and let $\xmin = \frac{1}{K}$.
    Then by \Cref{lem:driftBound}, including the fact that we consider the setting where we actually observe drift, we obtain for all $t \in \N$ before the phase ends that
    \begin{equation*}
        \expect{(X_t - X_{t + 1}) \cdot \indic{t < T \land \mathcal{A}_{t'}} \given \freq{t}{}{}} \geq \cstar \frac{1}{K \sqrt{n}} X_t \cdot \indic{t < T} \prob{\mathcal{A}_{t'} \given \freq{t}{}{}} .
    \end{equation*}
    By an application of the second statement of \Cref{thm:multiplicative-drift-concentration-with-failure} with $\gamma = c_2 \ln n$, noting that $X_0 \leq 1$ and thus $\expect{X_0 \cdot \indic{t < T}} / \xmin \leq K = \poly(n)$, we obtain that with probability at least $1 - n^{-c_2} - \bigO{n^{-3}}$, phase~$\kappa$ for position~$i$ ends in at most $c_1 K \sqrt{n} \ln n$ iterations.
    Via a union bound over all~$n$ positions, phase~$\kappa$ ends for all positions with probability at least $1 - n^{1 - c_2} - \bigO{n^{-2}}$ during the same number of iterations.

    \textbf{Showing that the assumptions for \Cref{lem:driftBound} are met.}
    Consider a position $i \in [n]$ and a phase $\kappa \in [0 .. \kappa^* - 2]$.
    Let $t' \in \N$ denote the first (random) iteration of phase~$\kappa$ for position~$i$, and for a sufficiently large constant $\cstop \in \R_{> 0}$, let $\tstop = \cstop K \sqrt{n}$ denote an upper bound on the length of this phase (that holds with high probability).
    Then $\mu_i^{(t')}(S_\kappa)=1$ and $\mu_i^{(t')}(K_\kappa) > 0$ hold by the definition of a phase.
    Assuming that $\mu^{(t')}_i(S_{\nu}) \ge (1-1/\kappa^*)^{3\kappa^*-3}/r$ also holds, which we prove below, \Cref{lem:adequate-biased-growth} yields that for a sufficiently large constants $\cstar \in \R_{> 3}$, with probability at least $1 - 3 \tstop n^{-\cstar / (3857 \cstop)}$, for the $\cstop K \sqrt{n} \ln n$ iterations following~$t'$, the ratio of cumulative probability mass in the upper frequencies is roughly the same as it is in the previous phase.

    Applying \Cref{lem:adequate-biased-growth} iteratively, assuming its conditions to be met for now, we obtain for all $\nu \in [\kappa + 1 .. \kappa^* - 1]$ and all $t \in [t' .. t' + \tstop]$, that
    \begin{equation*}
        \frac{\mu_{i}^{(t' + t)}(S_{\nu + 1})}{\mu_{i}^{(t' + t)}(S_\nu)}
        \geq \left(1 - \frac{1}{\kappa^*}\right)^{3 (\kappa^* - 1)} \frac{\mu_{i}^{(0)}(S_{\nu + 1})}{\mu_{i}^{(0)}(S_\nu)}
    \end{equation*}
    with probability at least $1 - 3 \kappa^* \tstop n^{-\cstar / (3857 \cstop)}$.

    Note that since in iteration~$0$, all frequencies are initialized to~$\frac{1}{r}$, and by \cref{eq:intervalHierarchy}, we see that $\mu_{i}^{(0)}(S_{\nu + 1}) / \mu_{i}^{(0)}(S_\nu)$ is bounded from below by
    \begin{equation*}
        \frac{r - 1 - \ell_{\nu + 1}}{r - 1 - \ell_\nu}
        = \frac{\ceil*{\left(\frac{2}{3}\right)^{\nu + 1} (r - 1)}}{\ceil*{\left(\frac{2}{3}\right)^\nu (r - 1)}}
        \geq \frac{\left(\frac{2}{3}\right)^{\nu + 1} (r - 1)}{\left(\frac{2}{3}\right)^\nu (r - 1) + 1}
        \geq \frac{2}{5} .
    \end{equation*}
    Thus, $\mu_{i}^{(t' + \tstop)}(S_{\nu + 1}) / \mu_{i}^{(t' + \tstop)}(S_\nu) \geq \frac{2}{5}$.
    This shows that the non-trivial assumption in \Cref{lem:driftBound} including the constant~$\cdrift$ is met.

    It remains to show that the assumption $\mu^{(t')}_i(S_{\nu}) \ge (1-1/\kappa^*)^{3\kappa^*-3}/r$ from \Cref{lem:adequate-biased-growth} is satisfied.
    This follows inductively from an application of \Cref{lem:adequate-biased-growth} for all preceding phases.
    The base case is satisfied since all frequencies are initialized with~$\frac{1}{r}$ by the definition of the \rcga.
    The other cases follow from \Cref{lem:contribution-random-walk-steps-kappastar}, using the induction hypothesis for the lower bound on the probability mass in~$S_\nu$, since $(1 - 1/\kappa^*)^{3(\kappa^* - 1)} \geq \eulerE[-3]$.
    Thus, since probability mass in~$S_\nu$ can never decrease due to biased steps, with probability at least $1 - 2 \kappa^* n^{-\cstar / (3857 \cstop)}$, each phase reduces the probability mass in~$S_\nu$ by at most a $1/\kappa^*$-fraction.
    This proves this claim.

    \textbf{Concluding.}
    We are left with adding up all the failure probabilities of all the cases discussed above.
    The bound of the multiplicative drift theorem holds per phase (and for all positions) with probability at least $1 - n^{1 - c_2} - \bigO{n^{-2}}$.
    Since we have at most~$\kappa^*$ phases, the probability of one phase taking longer is at most $\kappa^* (n^{1 - c_2} + \bigO{n^{-2}}) = \bigO{\frac{1}{n}}$, since~$c_2$ is sufficiently large and since $\kappa^* = \bigO{\log n}$.

    The failure probability of applying \Cref{lem:adequate-biased-growth} is per phase and per position at most $3 \kappa^* \tstop n^{-\cstar / (3857 \cstop)}$.
    A union bound over all at most~$\kappa^*$ phases and all~$n$ positions yields, since~$\cstar$ is sufficiently large and since $\tstop = \poly(n)$, a total failure probability of $\bigO{n^{-3}}$.
    Note that this failure probability also accounts for~$q$ in the application of \Cref{thm:multiplicative-drift-concentration-with-failure}.

    Last, the application of \Cref{lem:contribution-random-walk-steps-kappastar} adds a failure probability of at most $2 \kappa^* n^{-\cstar / (3857 \cstop)}$ for each phase and each position.
    Similar to the previous paragraph, this adds up to a total failure probability of $\bigO{\frac{1}{n}}$.

    Adding up these three total failure probabilities concludes the proof.
\end{proof}

\end{document}